\documentclass{article}
\pdfoutput=1
\PassOptionsToPackage{square,numbers,sort&compress}{natbib}
\usepackage[main,preprint]{neurips_2026}
\usepackage[utf8]{inputenc}
\usepackage{times}
\usepackage{amsmath,amssymb,amsthm}
\usepackage{algorithm,algpseudocode}
\usepackage{xspace}
\usepackage{graphicx,multirow,booktabs}
\usepackage[most]{tcolorbox}
\usepackage[small]{caption}
\usepackage[dvipsnames]{xcolor}
\usepackage{hyperref}
\usepackage{url}
\hypersetup{colorlinks=true,urlcolor=magenta,citecolor=brown,linkcolor=red,breaklinks=true}
\theoremstyle{plain}
\newtheorem{theorem}{Theorem}
\newtheorem{lemma}{Lemma}
\newtheorem{definition}{Definition}
\newtheorem{corollary}{Corollary}
\newtheorem{innercustomthm}{Theorem}
\newenvironment{customthm}[1]{\renewcommand\theinnercustomthm{#1}\innercustomthm}{\endinnercustomthm}
\newcommand{\method}{MAQEE\xspace}
\newcommand{\methodfull}{Mutual Adaptive Quantization with Early Exiting\xspace}
\newcommand{\vparagraph}[1]{
  {\noindent \textbf{#1}}
}
\newtcolorbox{takeaway}{colback=gray!10,colframe=gray!50,boxrule=0.5pt,arc=3pt,left=6pt,right=6pt,top=6pt,bottom=6pt}
\title{Amortized-Precision Quantization for Early-Exit Vision Transformers}
\author{
Rui Fang\thanks{Equal contribution.}\quad Hsi-Wen Chen\textsuperscript{*}\quad Ming-Syan Chen\\
National Taiwan University\\
\texttt{\{rfang, hwchen\}@arbor.ee.ntu.edu.tw}\\
\texttt{mschen@ntu.edu.tw}
}
\begin{document}
\maketitle
\begin{abstract}
{Vision Transformers (ViTs) achieve strong performance across vision tasks, yet their deployment with low-precision early exiting remains fragile. Existing quantization methods assume static full-depth execution, making them unstable when exit decisions are perturbed by quantization noise, which can amplify errors along dynamic inference paths. In this paper, we introduce \textbf{Amortized-Precision Quantization (APQ)}, a utilization-aware formulation that {accounts for layer utilization when modeling quantization error} and reveals \textit{depth--precision} trade-offs. Building on APQ, we propose \textbf{Mutual Adaptive Quantization with Early Exiting (MAQEE)}, a bi-level framework that jointly optimizes exit thresholds and bit-widths under explicit risk control to improve inference stability. MAQEE establishes a superior \textit{Pareto frontier} in the accuracy--efficiency trade-off across classification, detection, and segmentation{,}  {with comparisons at predefined task-specific performance targets. On ImageNet, it achieves 79.41\% Top-1 accuracy while reducing BOPs by 10.9\% relative to utilization-only allocation. Measurements of the deployed models on an RTX 5090 show GPU throughput improvements of $2.26\times$ on ImageNet and $3.06\times$ on CIFAR-100 over full-depth FP32 inference}.} {Source code is available at \url{https://github.com/Rui-Fang/MAQEE}.}
\end{abstract}

\section{Introduction}
\label{sec:introduction}
Vision Transformers (ViTs)~\citep{dosovitskiy2020image} achieve strong performance across classification~\citep{chen2021crossvit}, detection~\citep{carion2020end}, and segmentation~\citep{gao2021utnet}, but their high computational cost limits deployment on edge devices~\citep{zheng2023lightweight,shang2024quantized}. Quantization~\citep{courbariaux2015binaryconnect} improves efficiency by reducing weight and activation precision. Early approaches adopt \textit{Fixed-Precision Quantization (FPQ)}~\citep{krishnamoorthi2018quantizing,jacob2018quantization,yang2019quantization}, but uniform bit-width assignment ignores the heterogeneous sensitivity of ViT layers to quantization noise~\citep{liu2021layer,tai2024mptq}. \textit{Mixed-Precision Quantization (MPQ)}~\citep{xiao2023patch,jeon2024usdn} partially addresses this issue by assigning layer-wise bit-widths based on sensitivity, yet remains designed for static, full-depth inference.

However, modern ViT architectures increasingly rely on \textit{dynamic inference}, such as Early Exiting (EE)~\citep{teerapittayanon2016branchynet,laskaridis2021adaptive,xu2023lgvit}, where lightweight exit heads enable input-dependent termination to reduce latency. This creates a mismatch with conventional quantization: MPQ calibrates precision under a fixed execution path, whereas EE induces stochastic, input-dependent depths~\citep{rahmath2024early}.\footnote{Mixture of Experts (MoE)~\citep{riquelme2021scaling,hwang2023tutel} is another dynamic inference paradigm, but unlike EE, it selects a fixed sub-network per input rather than changing the executed depth. Thus, EE is more sensitive to quantization error.} As shown in the left panel of Fig.~\ref{fig:maqee}, quantization noise can perturb confidence estimates for early exiting, leading to either \textit{premature} or \textit{delayed} exits. Our preliminary results show this effect: quantizing EE-trained models causes up to a 50\% accuracy drop, while applying EE to already quantized models makes it difficult to properly train exit heads, resulting in limited efficiency gains. {Low-precision weights limit the flexibility needed to learn reliable exit decisions.}

\begin{figure*}[t]
    \centering
    \includegraphics[width=\textwidth]{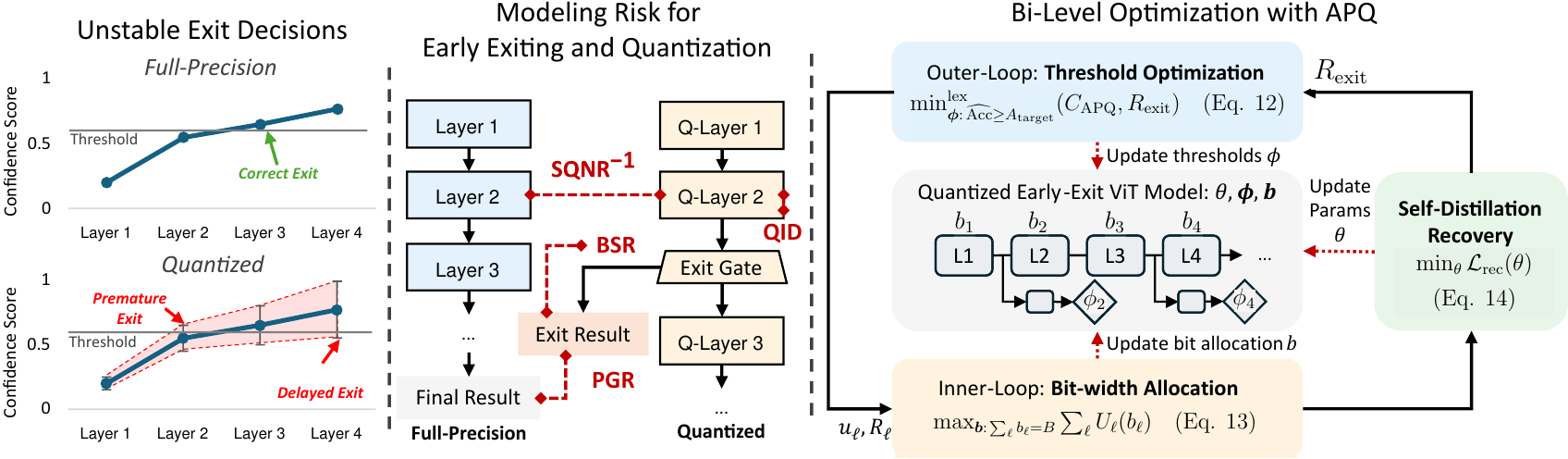} 
    \caption{\textbf{Overview of \method.} \textit{Left}: Quantization error perturbs early-exit decisions, causing either premature or delayed exits. \textit{Middle}: Risk modeling for early exiting and quantization, including performance gap risk (PGR), boundary sensitivity risk (BSR), inverse SQNR, and quantization-induced drift (QID). \textit{Right}: \method solves Amortized-Precision Quantization (APQ) via bi-level optimization, where exit thresholds and layer-wise bit-widths are jointly optimized, with self-distillation for accuracy recovery.}
\label{fig:maqee}
\end{figure*}

Motivated by these observations, we propose \textbf{Amortized-Precision Quantization (APQ)}, a quantization paradigm tailored to early-exit inference. Unlike conventional approaches that assume a fixed execution path, APQ treats precision as a resource whose impact depends on dynamic layer utilization. This induces a utilization-aware depth--precision trade-off: precision allocation should account for how often each layer is reached, how its quantization error affects subsequent inference paths, and how it may mislead exit decisions. Our theoretical analysis shows that static precision allocation fails to capture these effects, explaining the robustness degradation observed when quantization is naively combined with early exiting.

To effectively optimize APQ, we introduce \textbf{\methodfull (\method)}, a bi-level optimization framework that jointly optimizes exit thresholds and layer-wise bit-widths, as illustrated in the middle and right panels of Figure~\ref{fig:maqee}. \method models instability from both exit-decision uncertainty and representation distortion, using four risk terms to capture performance degradation, boundary sensitivity, quantization noise, and changes in residual-update magnitude. These risks are aggregated into a unified risk score that guides the alternating optimization of thresholds and bit-widths within a bi-level iterative procedure. With lightweight self-distillation and efficient warm-starting, \method improves low-precision recovery and accelerates convergence.

{\textbf{Our contributions are summarized as follows.} First, we formalize the incompatibility between quantization and early exiting, and propose Amortized-Precision Quantization (APQ) to characterize inference complexity under dynamic execution. Second, we propose \methodfull (\method), a bi-level framework that jointly optimizes exit thresholds and bit-width allocation through risk-aware control. Third, we theoretically explain why static FPQ and MPQ are mismatched with early exiting and show that \method enables more stable utilization-aware optimization. Finally, experiments across three vision tasks show that \method forms a superior \textit{Pareto frontier} in the accuracy--efficiency trade-off{. Comparisons with joint quantization--early-exit methods and simple allocation strategies assess the benefit of risk-aware optimization}, while  {GPU}  {measurements}  {establish}  {the}  {efficiency} of the  {deployed}  {models}.}
\section{Related Work}
\label{sec:related}
{Vision Transformers (ViTs)~\citep{dosovitskiy2020image} capture long-range dependencies through self-attention~\citep{vaswani2017attention,raghu2021vision}, but their high computation and memory costs limit deployment in real-time or resource-constrained settings~\citep{zheng2023lightweight,shang2024quantized}. Quantization reduces this burden with low-bit weights and activations~\citep{courbariaux2015binaryconnect,krishnamoorthi2018quantizing}. Fixed-Precision Quantization (FPQ)~\citep{jacob2018quantization,yang2019quantization} applies uniform bit-widths but ignores heterogeneous layer sensitivity in ViTs~\citep{liu2021layer,tai2024mptq}, leading to weak robustness and suboptimal accuracy--efficiency trade-offs~\citep{gholami2022survey}. Mixed-Precision Quantization (MPQ)~\citep{xiao2023patch,jeon2024usdn}, along with PTQ~\citep{liu2021post,ding2022towards,li2023repq,shi2024p} and QAT~\citep{wang2025domain,nagel2022overcoming,li2022q}, improves layer-wise precision allocation, but still assumes a static full-depth inference path. In parallel, dynamic inference adapts computation to input complexity by activating only part of the network~\citep{chen2025moequant,riquelme2021scaling,hwang2023tutel}. Early Exiting (EE)~\citep{xin2020deebert,schuster2022confident} attaches auxiliary classifiers to intermediate layers so confident samples terminate early, with ViT-based extensions showing promising efficiency gains~\citep{bakhtiarnia2021multi,xu2023lgvit,rahmath2024early}. However, quantization noise can disrupt confidence estimation and cause premature or delayed exits. McQueen~\citep{saxena2023mcqueen}  {learns layer-wise precision under a global budget in early-exit networks.} QuEE~\citep{regol2024predicting}  {selects}  {among}  {classifiers}  {with}  {different}  {exit}  {depths}  {and quantization levels. MAQEE jointly optimizes precision and exit thresholds using layer utilization and quantization risk, then deploys the resulting mixed-precision allocation}. {We introduce \textit{Amortized-Precision Quantization} to analyze how quantization and early exiting affect each other.}}
\section{Problem Formulation}
\label{sec:problem}
This section introduces \textbf{Amortized-Precision Quantization (APQ)} and formulates its corresponding optimization problem. A detailed notation table is provided in Appendix~\ref{apx:notation}.

\vparagraph{Quantization.}
Quantization replaces floating-point operations with low-bit integer arithmetic to reduce memory footprint and inference cost~\citep{jacob2018quantization}. Given bit-width $b$, a scale factor $s$, and a zero-point $t$, values are mapped to a $b$-bit integer range, e.g., $[-2^{b-1},2^{b-1}-1]$ for signed quantization, and dequantization approximates the original value by $s(x^q-t)$. Existing strategies include \textit{Fixed-Precision Quantization (FPQ)}~\citep{jacob2018quantization,yang2019quantization}, which uses a uniform bit-width, and \textit{Mixed-Precision Quantization (MPQ)}~\citep{xiao2023patch,jeon2024usdn}, which assigns sensitivity-aware layer-wise bit-widths.

Consider a network with $L$ layers and a set of available bit-widths $\mathcal{B}$, e.g., $\{2,4,8\}$. Each layer $\ell$ is assigned a bit-width $b_\ell\in\mathcal{B}$, giving an allocation $\boldsymbol{b}=\{b_\ell\}_{\ell=1}^L$. Let $c_\ell(b_\ell)$ denote the computational cost of layer $\ell$ under bit-width $b_\ell$. We measure cost by \textit{bit operations (BOPs)}, commonly defined as $\mathrm{FLOPs}\times B_w\times B_a$, where $B_w$ and $B_a$ are the weight and activation bit-widths~\citep{shang2024quantized,chen2025moequant}. Since conventional FPQ and MPQ assume that all inputs traverse all $L$ layers, their expected complexity depends only on the static allocation $\boldsymbol{b}$, following a deterministic full-depth path.

\vparagraph{Early Exiting.} 
Early exiting shortens inference by attaching auxiliary classifiers, or exit heads, to intermediate layers, allowing samples to terminate once predictions become sufficiently confident~\citep{schuster2022confident}. This induces \emph{stochastic, input-dependent inference depths}: easy samples exit early, while harder ones propagate deeper. {Let $\mathcal E\subseteq\{1,\ldots,L-1\}$ denote the intermediate exit layers.} Formally, let $z_\ell(x)$ denote the {full-precision logits} at layer $\ell$, and define the exit confidence as $\rho_\ell(x):=\max_j \mathrm{softmax}(z_\ell(x))_j$. An input exits at the first {exit layer $\ell\in\mathcal E$} such that $\rho_\ell(x)\geq \phi_\ell$, where $\phi_\ell\in[0,1]$ is the exit threshold; otherwise, it continues until layer $L$, which serves as the fallback exit. The exit policy is parameterized by {$\boldsymbol{\phi}=\{\phi_e\}_{e\in\mathcal E}$}.

To quantify stochastic execution under early exiting, we define the utilization indicator $u_\ell(x,\boldsymbol{\phi})=\mathbf{1}\{x \text{ reaches layer }\ell \mid \boldsymbol{\phi}\}$, and its expectation over the data distribution $\mathcal{D}$, $u_\ell(\boldsymbol{\phi})=\mathbb{E}_{x\sim\mathcal{D}}[u_\ell(x,\boldsymbol{\phi})]$, which measures the probability that layer $\ell$ is activated. The objective of early exiting is to reduce computation cost while preserving accuracy; however, its exit behavior fundamentally alters how quantization noise accumulates and propagates across layers.

\vparagraph{Amortized-Precision Quantization (APQ).}
Conventional quantization assumes a fixed full-depth inference path, where all layers are executed for every input. Under early exiting, however, execution depth becomes stochastic and input-dependent, so quantization noise is exposed and propagated only along realized inference paths. This creates a mismatch for static bit allocation: layers with different execution frequencies may have different effective precision requirements. To address this, we propose \emph{Amortized-Precision Quantization (APQ)}, which accounts for layer-wise execution frequency when allocating precision under early exiting.

\begin{definition}[Amortized-Precision Quantization (APQ)]\label{def:apq}
Consider a network with $L$ layers, where each layer $\ell$ is assigned a bit-width $b_\ell \in \mathcal{B}$ from the available bit-width set $\mathcal{B}$. The amortized computational complexity of a bit allocation $\boldsymbol{b}$ under an early-exit policy $\boldsymbol{\phi}$ is defined as
\begingroup
\begin{equation}
C_{\mathrm{APQ}}(\boldsymbol{b}, \boldsymbol{\phi}) = \sum_{\ell=1}^L
\mathbb{E}_{x \sim \mathcal{D}} \!\left[ u_\ell(x, \boldsymbol{\phi}) \cdot c_\ell(b_\ell) \right] 
= \sum_{\ell=1}^L u_\ell(\boldsymbol{\phi}) \cdot c_\ell(b_\ell),
\end{equation}
\endgroup
where $c_\ell(b_\ell)$ is the cost of layer $\ell$ under bit-width $b_\ell$, $u_\ell(x,\boldsymbol{\phi})\in\{0,1\}$ indicates whether layer $\ell$ is executed for input $x$, and $u_\ell(\boldsymbol{\phi})=\mathbb{E}_{x\sim\mathcal{D}}[u_\ell(x,\boldsymbol{\phi})]$ is its utilization under the exit policy.
\end{definition}

Since $u_\ell(\boldsymbol{\phi})$ depends on exit thresholds, APQ naturally leads to a joint optimization of bit-widths $\boldsymbol{b}$ and exit policy $\boldsymbol{\phi}$ to balance accuracy and amortized low-bit cost:
\begingroup
\begin{equation}\label{eqn:objective}
(\boldsymbol{b}^*, \boldsymbol{\phi}^*) 
= \arg\min_{\boldsymbol{b}, \boldsymbol{\phi}} \;
\mathbb{E}_{(x,y) \sim \mathcal{D}}
\!\left[ \mathcal{L}\big(f_{\theta,\boldsymbol{b},\boldsymbol{\phi}}(x), y\big) \right] 
+ \lambda C_{\mathrm{APQ}}(\boldsymbol{b},\boldsymbol{\phi}),
\end{equation}
\endgroup
where $\mathcal{L}$ denotes the task-specific loss, $f_{\theta,\boldsymbol{b},\boldsymbol{\phi}}$ denotes the quantized early-exit network, and $\lambda$ controls the accuracy--efficiency trade-off. More broadly, APQ provides a utilization-aware view for dynamic inference paradigms such as recursive Transformers~\cite{shen2022sliced}, MoE routing~\cite{zhou2022mixture}, and layer skipping~\cite{elhoushi2024layerskip}, where input-dependent computation induces non-uniform exposure of layers or modules to quantization noise.
\section{Why Static Precision Allocation Fails under Dynamic Inference}
\label{sec:theory}

This section analyzes the \emph{mutual interference} between quantization and early exiting under stochastic low-bit execution. Because early exiting induces input-dependent stopping depths, quantization affects not only predictions but also executed-depth distributions, while early exiting in turn changes each layer's exposure to quantization noise. This random-stopping view explains why static precision allocation is mismatched with dynamic inference.

\vparagraph{Quantization Perturbs the Exit Policy.}
For each {exit layer $\ell\in\mathcal E$}, let $\Delta_\ell(x)=\rho_\ell(x)-\phi_\ell$ denote the full-precision confidence margin relative to the exit threshold $\phi_\ell$. Let $\tilde{\rho}_{\ell,\boldsymbol{b}}(x)$ denote the confidence after quantization under bit allocation $\boldsymbol{b}$, and define the induced confidence perturbation as
\begingroup
\begin{equation}
\xi_{\ell,\boldsymbol{b}}(x)
=
\tilde{\rho}_{\ell,\boldsymbol{b}}(x)-\rho_\ell(x).
\end{equation}
\endgroup
We define the full-precision and quantized exit decisions as $e_\ell(x)=\mathbf{1}\{\rho_\ell(x)\geq\phi_\ell\}$ and $\tilde e_\ell(x)=\mathbf{1}\{\tilde{\rho}_{\ell,\boldsymbol{b}}(x)\geq\phi_\ell\}$, respectively. The two decisions can differ only when the perturbation crosses the confidence boundary: if $\tilde e_\ell(x)\neq e_\ell(x)$, then $|\Delta_\ell(x)|\leq |\xi_{\ell,\boldsymbol{b}}(x)|$. Thus, exit instability depends jointly on the margin distribution around the threshold and the magnitude of the quantization-induced confidence perturbation.

\begin{theorem}
\label{thm:margin_exit_instability}
{For an exit layer $\ell\in\mathcal E$,} assume that the confidence perturbation $\xi_{\ell,\boldsymbol{b}}(x)$ satisfies the tail bound
\begingroup
\begin{equation*}
\Pr\!\left(|\xi_{\ell,\boldsymbol{b}}(x)|>\tau_\ell\right)
\leq
q_\ell(\tau_\ell,b_\ell),
\end{equation*}
\endgroup
where $q_\ell(\tau_\ell,b_\ell)$ is non-increasing in $\tau_\ell$ and non-decreasing as $b_\ell$ decreases. Then, for any tolerance $\tau_\ell>0$, the mis-exit probability satisfies
\begingroup
\begin{equation*}
\Pr\!\left(\tilde e_\ell(x)\neq e_\ell(x)\right)
\leq
\Pr\!\left(|\Delta_\ell(x)|\leq\tau_\ell\right)
+
q_\ell(\tau_\ell,b_\ell).
\end{equation*}
\endgroup
\end{theorem}

Theorem~\ref{thm:margin_exit_instability} decomposes exit instability into a boundary-mass term and a quantization-perturbation tail term. The term $\Pr(|\Delta_\ell(x)|\leq\tau_\ell)$ captures samples near the confidence boundary, where small perturbations may flip the exit decision. The term $q_\ell(\tau_\ell,b_\ell)$ captures the probability that the quantization-induced perturbation exceeds the tolerated margin. As $b_\ell$ decreases, this tail term increases, yielding a larger upper bound on the exit-decision mismatch probability.

Since early exiting stops at the first confident layer, such layer-wise mismatches can shift the realized stopping-depth distribution, leading to premature exits that harm accuracy or delayed exits that reduce efficiency (see Corollary~\ref{cor:stopping_depth_perturbation} in Appendix~\ref{proof:stopping_depth_perturbation}).

\vparagraph{Random-Depth Amplifies Quantization Error.}
We next show that early exiting turns quantization from a fixed-depth approximation problem into a random-depth distortion problem. Let $h_\ell(x)$ and $\tilde h_{\ell,\boldsymbol{b}}(x)$ denote the full-precision and quantized hidden states at layer $\ell$. Suppose that, for $\ell=1,\ldots,L$, the layer-wise distortion satisfies
\begingroup
\begin{equation*}
\|\tilde h_{\ell,\boldsymbol{b}}(x)-h_\ell(x)\|
\leq
\gamma_{\ell-1}
\|\tilde h_{\ell-1,\boldsymbol{b}}(x)-h_{\ell-1}(x)\|
+
\varepsilon_\ell(b_\ell),
\end{equation*}
\endgroup
where $\varepsilon_\ell(b_\ell)$ denotes the local distortion introduced at layer $\ell$, $\gamma_{\ell-1}$ captures sensitivity from layer $\ell-1$ to layer $\ell$, and the input distortion is initialized as $\|\tilde h_{0,\boldsymbol{b}}(x)-h_0(x)\|=0$.

\begin{theorem}
\label{thm:random_depth_error}
For any realized stopping depth $k$, the accumulated representation error satisfies
\begingroup
\begin{equation*}
\|\tilde h_{k,\boldsymbol{b}}(x)-h_k(x)\|
\leq
\sum_{\ell=1}^{k}
\varepsilon_\ell(b_\ell)
\prod_{r=\ell+1}^{k}\gamma_{r-1},
\end{equation*}
\endgroup
where the empty product is defined as $1$. Under early exiting with random stopping depth $K_{\boldsymbol{\phi}}(x)$, the expected distortion satisfies
\begingroup
\begin{equation*}
\mathbb{E}_{x\sim\mathcal{D}}
\left[
\|\tilde h_{K_{\boldsymbol{\phi}}(x),\boldsymbol{b}}(x)
-
h_{K_{\boldsymbol{\phi}}(x)}(x)\|
\right]
\leq
\mathbb{E}_{x\sim\mathcal{D}}
\left[
\sum_{\ell=1}^{K_{\boldsymbol{\phi}}(x)}
\varepsilon_\ell(b_\ell)
\prod_{r=\ell+1}^{K_{\boldsymbol{\phi}}(x)}
\gamma_{r-1}
\right].
\end{equation*}
\endgroup

\end{theorem}

This theorem shows that quantization error under early exiting is path-dependent: its impact depends on local representation distortion, downstream propagation, and the probability that each layer is reached. Errors in frequently executed or shallow layers can have a larger amortized impact, while rarely executed layers may contribute less.

\vparagraph{Static Precision Allocation Is Mismatched.}
Conventional FPQ and MPQ allocate bit-widths according to static layer sensitivity, implicitly assuming that every layer is executed for every input. Under early exiting, however, the contribution of layer $\ell$ is weighted by its utilization probability
\begingroup
\begin{equation*}
u_\ell(\boldsymbol{\phi})
=
\Pr_{x\sim\mathcal{D}}\!\left(K_{\boldsymbol{\phi}}(x)\geq \ell\right).
\end{equation*}
\endgroup
Let $S_\ell$ denote the static quantization sensitivity of layer $\ell$. Under APQ, the relevant marginal importance becomes the utilization-weighted quantity $u_\ell(\boldsymbol{\phi})S_\ell$, rather than $S_\ell$ alone.

\begin{theorem}
\label{thm:mpq_mismatch}
Assume that static MPQ ranks layers by their static sensitivities $S_\ell$, while APQ ranks layers by their utilization-weighted sensitivities $u_\ell(\boldsymbol{\phi})S_\ell$. If there exist two layers $i$ and $j$ such that
\begingroup\begin{equation*}
S_i>S_j
\quad\text{but}\quad
u_i(\boldsymbol{\phi})S_i<u_j(\boldsymbol{\phi})S_j,
\end{equation*}\endgroup
then the static MPQ ranking is inconsistent with the APQ ranking. Moreover, under equal marginal bit-cost increments and marginal APQ benefit proportional to $u_\ell(\boldsymbol{\phi})S_\ell$, the static MPQ allocation can be suboptimal for the APQ objective.
\end{theorem}
Theorem~\ref{thm:mpq_mismatch} shows that static quantization can be suboptimal even with accurate full-depth sensitivity estimates, because it ignores layer utilization under early exiting. This motivates incorporating layer utilization into risk estimation and cost evaluation when jointly optimizing precision and exit thresholds.

\begin{takeaway} 
\textbf{Takeaway.}
The above results yield two main observations. First, Theorem~\ref{thm:margin_exit_instability} shows that quantization can flip exit decisions near the boundary, motivating the decision-level boundary sensitivity risk (BSR). Second, Theorem~\ref{thm:random_depth_error} shows that quantization error is path-dependent under early exiting, motivating the representation-level quantization-induced drift (QID) and utilization-aware bit-width allocation. Together, these results explain the objective mismatch in Theorem~\ref{thm:mpq_mismatch}: static MPQ ignores utilization, whereas APQ calls for a framework that jointly optimizes bit-width allocation and exit-threshold adaptation.
\end{takeaway}

\section{\method}
\label{sec:method}
To optimize APQ effectively, we need to balance dynamic execution depth and layer-wise precision under quantization. Exit thresholds determine which layers are executed, while bit-width allocation controls how quantization error affects exit decisions and downstream representations. This coupling naturally gives rise to \textbf{\methodfull (\method)}, a bi-level optimization framework that jointly optimizes exit thresholds and layer-wise bit-width allocation. Detailed pseudocode is given in Appendix~\ref{apx:code}.

\subsection{Bi-Level Optimization with APQ}
\label{sub:question}

Since APQ aims to reduce expected amortized complexity while preserving prediction accuracy, the problem is naturally formulated as a \textbf{bi-level optimization} over $(\boldsymbol{b}, \boldsymbol{\phi})$.
Recall that $u_\ell(\boldsymbol{\phi})$ denotes the probability that inference reaches layer $\ell$.
We denote the expected task loss under APQ as follows.
\begingroup
\begin{equation}
\mathcal{L}(\boldsymbol{b}, \boldsymbol{\phi})
:= \mathbb{E}_{(x,y)\sim\mathcal{D}}
\Big[ \mathcal{L}\big(f_{\theta,\boldsymbol{b},\boldsymbol{\phi}}(x), y\big) \Big].
\end{equation}
\endgroup

\begin{samepage}
\vparagraph{Outer-loop (Exit-threshold optimization).}
Let {$\boldsymbol{\phi}=\{\phi_e\}_{e\in\mathcal E}$} denote the confidence thresholds of the exit heads.
The outer problem optimizes these thresholds to balance prediction accuracy and {amortized computational cost}:
\begingroup
\begin{equation}
\begin{aligned}
\boldsymbol{\phi}^*
= \arg\min_{\boldsymbol{\phi}} \;&
\mathcal{L}\big(\boldsymbol{b}^*(\boldsymbol{\phi}), \boldsymbol{\phi}\big)
+ \lambda_{\mathrm{outer}} \,
C_{\mathrm{APQ}}(\boldsymbol{b}^*(\boldsymbol{\phi}), \boldsymbol{\phi}),
\end{aligned}
\label{eq:outer}
\end{equation}
\endgroup
\par
\end{samepage}
where $\lambda_{\mathrm{outer}}$ controls the trade-off between accuracy and {amortized computational cost}, and
$\boldsymbol{b}^*(\boldsymbol{\phi})$ denotes the optimal bit-width allocation induced by the exit policy $\boldsymbol{\phi}$ through the inner loop.

\vparagraph{Inner-loop (Bit-width allocation).}
For fixed exit thresholds $\boldsymbol{\phi}$, the bit-width allocation
$\boldsymbol{b}=\{b_\ell\}_{\ell=1}^L$ is optimized to minimize amortized complexity while preserving accuracy:
\begingroup
\begin{equation}
\boldsymbol{b}^*(\boldsymbol{\phi})
= \arg\min_{\boldsymbol{b}} 
\mathcal{L}(\boldsymbol{b}, \boldsymbol{\phi})
+ \lambda_{\mathrm{inner}} \,
C_{\mathrm{APQ}}(\boldsymbol{b}, \boldsymbol{\phi}),
\label{eq:inner}
\end{equation}
\endgroup
where $\lambda_{\mathrm{inner}}$ controls the strength of compression.

This bi-level formulation explicitly couples quantization with early exiting. The exit policy $\boldsymbol{\phi}$ determines layer utilization $u_\ell(\boldsymbol{\phi})$, while the bit allocation $\boldsymbol{b}$ affects confidence estimates and may alter realized exits. This coupling reflects the objective mismatch in Theorem~\ref{thm:mpq_mismatch}: static FPQ and MPQ ignore stochastic utilization and cannot fully capture the APQ objective. By jointly optimizing $\boldsymbol{\phi}$ and $\boldsymbol{b}$, APQ balances accuracy and efficiency under dynamic inference.

Since directly solving this bi-level problem is difficult due to discrete exit decisions and quantization-induced representation shifts, we introduce risk measures to guide exit-threshold adaptation and bit-width reallocation.\footnote{The NP-hardness of APQ is proved in Theorem~\ref{thm:np}, with details in Appendix~\ref{proof:np}.} We then present the iterative optimization procedure with lightweight self-distillation for accuracy recovery. {The empirical solver uses an accuracy constraint for threshold selection and a fixed bit budget for allocation; Appendix~\ref{apx:solver} gives the implementation details.}

\subsection{Modeling Risk for Early Exiting and Quantization}
\label{sec:risk}

As shown in Section~\ref{sec:theory}, APQ instability arises from both decision uncertainty and representation distortion. We therefore introduce four complementary risk scores to capture these effects. {We write $\tilde\rho_\ell(x)=\tilde\rho_{\ell,\boldsymbol b}(x)$ for the confidence under the current bit allocation.}

\vparagraph{Boundary Sensitivity Risk (BSR).}
To characterize the instability of exit decisions under quantization noise, we introduce Boundary Sensitivity Risk {at each exit layer $\ell\in\mathcal E$}:
\begingroup
\begin{equation}
\label{eq:bsr}
\mathrm{BSR}_\ell
=
\mathbb{E}_{x\sim \mathcal{D}}
\Big[
{u_\ell(x,\boldsymbol\phi)}\,\mathbf{1}\big\{|{\tilde\rho_\ell(x)}-\phi_\ell|{<} \tau_\ell \big\}
\Big],
\end{equation}
\endgroup
where $\tau_\ell$ is a tolerance margin. BSR measures the fraction of samples {that reach the exit and lie} near the exit boundary. {This measure is motivated by} the boundary-mass term in Theorem~\ref{thm:margin_exit_instability}. A high BSR indicates that small quantization perturbations can easily flip exit outcomes.

\vparagraph{Quantization-Induced Drift (QID).}
Let $h_\ell(x)$ and $\tilde h_\ell(x)$ denote the full-precision and quantized activations at layer $\ell$, respectively. To capture changes in residual-update magnitude under quantization, we introduce Quantization-Induced Drift:
\begingroup
\begin{equation}\label{eq:qid}
\mathrm{QID}_\ell=\left[\log\frac{\tilde{\sigma}_\ell+\epsilon}{\sigma_\ell+\epsilon}\right]^2,
\end{equation}\endgroup
where $\sigma_\ell$ and $\tilde{\sigma}_\ell$ denote the magnitudes of the full-precision and quantized residual updates, respectively. Larger QID indicates a greater relative change in residual-update magnitude.

Beyond the two APQ-specific risks above, we further incorporate two standard scores from early exiting and quantization to capture task-level degradation and hidden-state distortion.

\vparagraph{Performance Gap Risk (PGR).}
We quantify early-exit degradation via Performance Gap Risk~\citep{jazbec2024fast}. For each exit head $\ell$, it is defined as
\begingroup
\begin{equation}\label{eq:pgr}
\mathrm{PGR}_\ell
=\mathbb{E}_{(x,y)\sim\mathcal D}
\left[\mathcal L(\hat y_\ell(x),y)-\mathcal L(\hat y_L(x),y)\right],
\end{equation}
\endgroup
where $\mathcal L$ is the task-specific objective, $\hat y_\ell(x)$ is the prediction from the exit head at layer $\ell$, and $\hat y_L(x)$ is the full-depth prediction. PGR measures the excess loss incurred when exiting at layer $\ell$ instead of using the full-depth output.

\vparagraph{Inverse Signal-to-Quantization-Noise Ratio (SQNR$^{-1}$).}
As a supporting representation-level risk, we also use the inverse signal-to-quantization-noise ratio~\citep{tai2024mptq}, defined as
\begingroup
\begin{equation}
\mathrm{SQNR}^{-1}_\ell
=
\mathbb{E}_{x\sim\mathcal{D}}
\left[
\frac{\|h_\ell(x)-\tilde{h}_\ell(x)\|^2}{\|h_\ell(x)\|^2}
\right].
\end{equation}
\endgroup
It quantifies the relative power of quantization noise and complements QID by measuring pointwise representation distortion.

\vparagraph{Overall Risk Metric.}
We combine the above risks into a per-layer instability score:
\begingroup
\begin{equation}
\label{eq:risk}
R_\ell
=
{\sum_{j\in\{\mathrm{PGR},\mathrm{BSR},\mathrm{QID},\mathrm{SQNR}\}}w_j\mathcal N(r_j)_\ell,}
\end{equation}
\endgroup
{Here $r_j=(r_{j,1},\ldots,r_{j,L})$ collects risk component $j$ across layers, $w_j$ is its weight, and $\mathcal N$ normalizes it across layers.} PGR and BSR capture decision-level risk, while QID and SQNR$^{-1}$ capture representation-level risk. {Appendix~\ref{apx:solver} details the risk estimators, weights, and mapping from exits to blocks.}

\subsection{Iterative Optimization for APQ}
\label{sec:iterative}

Building on the proposed risk formulation, we develop a unified iterative optimization framework for APQ. Each iteration alternates among three stages: {(i) applying a recovery step to the model parameters $\theta$; (ii) optimizing the exit thresholds $\boldsymbol\phi$ under a fixed bit allocation $\boldsymbol b$; and (iii) reallocating layer-wise bit-widths according to the normalized risks.}

\vparagraph{Exit-Threshold Optimization (Outer-loop).}
Under a fixed bit allocation $\boldsymbol b$, we rank exit policies that meet the accuracy target by computational cost and exit risk:
\begingroup
\begin{equation}\label{eq:outer-opt}
\boldsymbol\phi^\star\in
\operatorname*{arg\,min}^{\mathrm{lex}}_{\boldsymbol\phi:\,\widehat{\mathrm{Acc}}(\boldsymbol b,\boldsymbol\phi)\ge A_{\mathrm{target}}}
\Bigl(C_{\mathrm{APQ}}(\boldsymbol b,\boldsymbol\phi),\,
R_{\mathrm{exit}}(\boldsymbol b,\boldsymbol\phi)\Bigr).
\end{equation}\endgroup
Here $R_{\mathrm{exit}}$ is the ordered pair of boundary sensitivity and exit prediction loss. We compare policies first by cost, then by boundary sensitivity, and finally by prediction loss. We search for thresholds via {grid search with local refinement}. Confidence scores {$\tilde\rho_\ell(x)$} are obtained using {one forward pass over the held-out calibration set}, after which candidate thresholds can be efficiently evaluated without retraining~\citep{jazbec2024fast}.

\vparagraph{Bit-width Allocation (Inner-loop).}
Given optimized early-exit thresholds $\boldsymbol\phi$, we allocate per-layer bit-widths $\boldsymbol b$ under a total-bit budget $B$ by evaluating the estimated risk reduction $\Delta R_\ell(b_\ell)$ from assigning an additional bit to layer $\ell$. Under the current exit distribution, the corresponding increase in amortized computational cost is
\[
\Delta\mathrm{BOPs}_\ell(\boldsymbol b,\boldsymbol\phi)
=C_{\mathrm{APQ}}(\boldsymbol b^{+\ell},\boldsymbol\phi)
-C_{\mathrm{APQ}}(\boldsymbol b,\boldsymbol\phi),
\]
where $\boldsymbol b^{+\ell}$ increases only $b_\ell$ by one bit. Let $U_\ell(b_\ell)$ denote the accumulated estimated risk reduction from increasing layer $\ell$ from the minimum bit-width to $b_\ell$. We select
\begin{equation}\label{eq:psi_ell_bits}
\boldsymbol b^\star\in\arg\max_{\boldsymbol b\in\mathcal B^L}\sum_{\ell=1}^{L}U_\ell(b_\ell)
\quad\text{s.t.}\quad\sum_{\ell=1}^{L}b_\ell=B,
\end{equation}
where $B=L\bar b$ for average weight precision $\bar b$. Dynamic programming accumulates estimated risk reduction and BOPs across layers, preferring lower BOPs when the estimated reductions are equal. Appendix~\ref{apx:solver} gives the risk-reduction estimates and solver details.

\vparagraph{Self-distillation Recovery.}
We use self-distillation to recover performance under low precision and dynamic execution. {The teacher parameters $\theta_{\mathrm T}$ remain fixed.}
The recovery objective jointly enforces task supervision and distributional alignment:
\begingroup
\begin{equation}\label{eq:recovery_loss}
\mathcal L_{\mathrm{rec}}(\theta)
=\sum_{e\in\mathcal E\cup\{L\}}\omega_e
\big[(1-\eta_e)\mathcal L_{\mathrm{CE},e}(\theta)
+\eta_e\mathcal L_{\mathrm{KL},e}(\theta;\theta_{\mathrm T})\big],
\end{equation}\endgroup
{where $\mathcal E$ is the set of intermediate exits, $\omega_e$ weights each classifier, and $\eta_e$ balances task supervision and teacher-to-student distillation.} By aligning the quantized early-exit student with its full-precision teacher, this step compensates for quantization-induced representation drift and stabilizes exit decisions.

\vparagraph{Warm-Started Iteration.}
To accelerate iterative optimization, we {initialize exit-threshold search from the previous solution}.

{Finally, MAQEE alternates between threshold updates and bit-allocation updates to balance accuracy and amortized computation. Threshold updates determine which layers are reached, while the resulting risks guide bit allocation. In contrast, static FPQ and MPQ ignore stochastic layer utilization and can therefore be suboptimal under early exiting. The search stops when an allocation repeats or the round budget is exhausted. We select the evaluated model with the lowest BOPs among those meeting the target accuracy.}

\clearpage
\section{Experiments}
\label{sec:experiment}
{In the following, we evaluate \method on three vision tasks and three backbones, comparing it with early-exit methods{,} quantization baselines{, joint methods, and simple precision-allocation strategies}.}

\subsection{Experimental Setup}

\vparagraph{Datasets.}
We evaluate \method on image classification (CIFAR-100~\citep{krizhevsky2009learning}, ImageNet~\citep{deng2009imagenet}), semantic segmentation (SceneParse150~\citep{zhou2017scene}), and object detection (MS COCO~\citep{lin2014microsoft}), reporting classification accuracy, {mean intersection over union (mIoU), and mAP@0.5}, respectively. Experiments are conducted on three backbones: DeiT~\citep{touvron2021training}, ViT~\citep{dosovitskiy2020image}, and Swin~\citep{liu2021swin}. For classification and segmentation, we use standard cross-entropy loss; for object detection, we use the bipartite matching loss~\citep{carion2020end}.

{\vparagraph{Baselines.}
We compare \method against representative early-exiting baselines, including ViT-EE~\citep{bakhtiarnia2021multi} and LGViT~\citep{xu2023lgvit}, as well as state-of-the-art quantization methods for ViTs. Specifically, we include fixed-precision quantization (FPQ) methods: RepQ~\citep{li2023repq} and ERQ~\citep{zhong2025towards}, and a mixed-precision quantization (MPQ) method: MPTQ~\citep{tai2024mptq}. {FP4/4 uses 4-bit weights in every layer, whereas MP4/4 varies weight precision across layers with an average of 4 bits. Both use 4-bit activations.}} {We also adapt McQueen and QuEE to CIFAR-100/ViT-B. McQueen uses an equivalent 4-bit budget; QuEE uses W3--W5 candidate paths. Appendix~\ref{apx:newexperiments} provides the additional protocols.}

\vparagraph{Evaluation Protocol.}
For a fair comparison~\citep{shang2024quantized}, we evaluate the accuracy--efficiency trade-off under two complementary settings.
\emph{(i) Controlled Performance.} We measure the exit layer $\hat{L}$ and bit operations $\hat{\mathrm{BOPs}}$ required to reach a predefined target performance, set to the quantization-only baseline without early exiting: 87\% for CIFAR-100, 79\% for ImageNet, 30\% for SceneParse150, and 55\% for MS COCO. Models that fail to meet the target are reported with $\hat{L}$ and $\hat{\mathrm{BOPs}}$ as N/A.
\emph{(ii) Standard Configuration.} Baseline models are evaluated using their default thresholds and quantization settings. {Under Standard Configuration}, we report the average exit layer $\bar{L}$, the average bit operations $\bar{\mathrm{BOPs}}$, and the task performance. {Bold entries highlight MAQEE. Dashes indicate quantities that are not applicable or not reported; N/A indicates that the target performance was not reached.} All results are averaged over 10 runs, with standard deviations less than 3\%. Additional details and hyperparameter settings are provided in Appendix~\ref{apx:implement}.

\subsection{Quantitative Results}

\begin{table}[t]
\centering\small
\renewcommand{\arraystretch}{1.00}
\captionsetup{font={small,color=black}}
\caption{Classification results on DeiT-B and ViT-L under Controlled Performance and Standard Configuration.}
\label{tab:deit}

\resizebox{\textwidth}{!}{\begin{tabular}{@{}cccccccccccccc@{}}
\toprule
\multirow{3}{*}{} & \multirow{3}{*}{\begin{tabular}[c]{@{}c@{}}Bits\\ (W/A)\end{tabular}} & \multirow{3}{*}{Quant} & \multirow{3}{*}{EE} & \multicolumn{5}{c}{CIFAR-100} & \multicolumn{5}{c}{ImageNet} \\ \cmidrule(l){5-14} 
 &  &  &  & \multicolumn{2}{c}{Controlled Perf.} & \multicolumn{3}{c}{Std. Config} & \multicolumn{2}{c}{Controlled Perf.} & \multicolumn{3}{c}{Std. Config} \\ \cmidrule(l){5-14} 
 &  &  &  & $\hat{L}\downarrow$ & $\hat{\text{BOPs}}\downarrow$ & $\bar{L}\downarrow$ & $\bar{\text{BOPs}} \downarrow$ & Acc.{ (\%)}$\uparrow$ & $\hat{L}\downarrow$ & $\hat{\text{BOPs}}\downarrow$ & $\bar{L}\downarrow$ & $\bar{\text{BOPs}} \downarrow$ & Acc.{ (\%)}$\uparrow$ \\ \midrule
\multirow{10}{*}{\rotatebox[origin=c]{90}{{DeiT-B}}} & \multirow{3}{*}{FP32} & \multirow{3}{*}{{--}} & {\textit{--}} & {\textit{--}} & {\textit{--}} & {\textit{12.00}} & \textit{1.80E13} & \textit{91.86} & {\textit{--}} & {\textit{--}} & {\textit{12.00}} & \textit{1.80E13} & \textit{83.10} \\ \cmidrule(l){4-14} 
 &  &  & ViT-EE & 7.20 & 1.84E13 & 8.10 & 1.97E13 & 89.24 & 7.89 & 1.76E13 & 8.15 & 1.97E13 & 80.51 \\
 &  &  & LGViT & 5.47 & 1.61E13 & 6.33 & 1.70E13 & 89.03 & 6.19 & 1.72E13 & 7.18 & 1.87E13 & 81.70 \\ \cmidrule(l){2-14} 
 & \multirow{4}{*}{FP4/4} & \multirow{2}{*}{RepQ} & ViT-EE & N/A & N/A & 9.91 & 3.59E11 & 84.79 & N/A & N/A & 11.42 & 3.85E11 & 70.58 \\
 &  &  & LGViT & N/A & N/A & 9.37 & 3.42E11 & 75.43 & N/A & N/A & 10.03 & 3.56E11 & 69.31 \\ \cmidrule(l){3-14} 
 &  & \multirow{2}{*}{ERQ} & ViT-EE & 8.77 & 3.23E11 & 9.28 & 3.35E11 & 87.86 & 11.79 & 3.91E11 & 10.86 & 3.72E11 & 77.59 \\
 &  &  & LGViT & N/A & N/A & 9.00 & 3.34E11 & 78.51 & N/A & N/A & 9.09 & 3.37E11 & 71.32 \\ \cmidrule(l){2-14} 
 & \multirow{3}{*}{MP4/4} & \multirow{2}{*}{MPTQ} & ViT-EE & 8.81 & 3.25E11 & 8.51 & 3.17E11 & 84.57 & 11.80 & 3.94E11 & 10.56 & 3.65E11 & 76.15 \\
 &  &  & LGViT & 11.35 & 3.93E11 & 7.76 & 3.07E11 & 83.15 & {12.00} & 4.09E11 & 8.86 & 3.34E11 & 73.58 \\ \cmidrule(l){3-14} 
 &  & \multicolumn{2}{c}{\textbf{MAQEE}} & \textbf{5.86} & \textbf{2.61E11} & \textbf{5.98} & \textbf{2.64E11} & \textbf{87.86} & \textbf{7.92} & \textbf{3.10E11} & \textbf{7.59} & \textbf{3.01E11} & \textbf{78.57}\\ \midrule
\multirow{10}{*}{\rotatebox[origin=c]{90}{ViT-L}} &
  \multirow{3}{*}{FP32} &
  \multirow{3}{*}{{--}} &
  {\textit{--}} &
  {\textit{--}} &
  {\textit{--}} &
  {\textit{24.00}} &
  \textit{6.31E13} &
  \textit{93.25} &
  {\textit{--}} &
  {\textit{--}} &
  {\textit{24.00}} &
  \textit{6.31E13} &
  \textit{85.30} \\ \cmidrule(l){4-14} 
 &                        &                       & ViT-EE & 11.75 & 3.11E13 & 12.94 & 3.41E13 & 92.40 & 12.43 & 3.36E13 & 16.01 & 4.37E13 & 82.86 \\
 &                        &                       & LGViT  & 8.34  & 2.26E13 & 10.85 & 2.93E13 & 92.53 & 9.95  & 2.69E13 & 12.52 & 3.38E13 & 82.02 \\ \cmidrule(l){2-14} 
 & \multirow{4}{*}{FP4/4} & \multirow{2}{*}{RepQ} & ViT-EE & N/A   & N/A     & 23.11 & 1.12E12 & 81.31 & N/A   & N/A     & 19.95 & 9.81E11 & 76.62 \\
 &                        &                       & LGViT  & 23.94 & 1.14E12 & 20.69 & 1.01E12 & 85.97 & N/A   & N/A     & 15.76 & 8.18E11 & 74.02 \\ \cmidrule(l){3-14} 
 &                        & \multirow{2}{*}{ERQ}  & ViT-EE & N/A   & N/A     & 22.58 & 1.09E12 & 82.34 & N/A   & N/A     & 21.53 & 1.05E12 & 58.89 \\
 &                        &                       & LGViT  & 21.06 & 1.02E12 & 20.74 & 1.01E12 & 85.59 & 22.20 & 1.08E12 & 17.53 & 8.95E11 & 76.52 \\ \cmidrule(l){2-14} 
 & \multirow{3}{*}{MP4/4} & \multirow{2}{*}{MPTQ}  & ViT-EE & N/A   & N/A     & 21.96 & 1.06E12 & 83.31 & N/A   & N/A     & 21.33 & 1.04E12 & 75.44 \\
 &                        &                       & LGViT  & 20.28 & 9.92E11 & 21.64 & 1.05E12 & 84.33 & 21.32 & 1.01E12 & 17.10 & 8.53E11 & 75.12 \\ \cmidrule(l){3-14} 
 &
   &
  \multicolumn{2}{c}{\textbf{MAQEE}} & \textbf{9.66} & \textbf{5.64E11} & \textbf{9.71} & \textbf{5.79E11} & \textbf{88.17} & \textbf{12.39} & \textbf{6.32E11} & \textbf{18.64} & \textbf{9.15E11} & \textbf{80.29}\\ \bottomrule
\end{tabular}}

\par
\begingroup

\captionsetup{font={small,color=black}}
\centering\small
\renewcommand{\arraystretch}{1.00}
\caption{Performance on semantic segmentation (ViT-B) and object detection ({DeiT-B}).}

\resizebox{\textwidth}{!}{\begin{tabular}{@{}cccccccccccc@{}}
\toprule
\multirow{3}{*}{} & \multirow{3}{*}{Method} & \multicolumn{5}{c}{Segmentation (SceneParse150)} & \multicolumn{5}{c}{Detection (MS COCO)} \\ \cmidrule(l){3-12} 
 &  & \multicolumn{2}{c}{Controlled Perf.} & \multicolumn{3}{c}{Std. Config} & \multicolumn{2}{c}{Controlled Perf.} & \multicolumn{3}{c}{Std. Config} \\ \cmidrule(l){3-12} 
 &  & $\hat{L}\!\downarrow$ & $\hat{\text{BOPs}}\!\downarrow$ & $\bar{L}\!\downarrow$ & $\bar{\text{BOPs}}\!\downarrow$ & {mIoU (\%)}$\uparrow$ & $\hat{L}\!\downarrow$ & $\hat{\text{BOPs}}\!\downarrow$ & $\bar{L}\!\downarrow$ & $\bar{\text{BOPs}}\!\downarrow$ & mAP@0.5{ (\%)}$\uparrow$ \\ \midrule
 & \textit{FP32} & {--} & {--} & {12.00} & 1.89E13 & 34.14 & {--} & {--} & {12.00} & 1.95E13 & 68.23 \\ \cmidrule(l){2-12} 
 & ViT-EE+ERQ & N/A & N/A & 8.10 & 3.17E11 & 27.34 & N/A & N/A & 8.72 & 3.34E11 & 40.42 \\
 & ViT-EE+{MPTQ} & 11.46 & 3.94E11 & 8.08 & 3.16E11 & 26.52 & N/A & N/A & 8.59 & 3.30E11 & 41.04 \\
 & \textbf{MAQEE} & \textbf{5.74} & \textbf{2.60E11} & \textbf{6.79} & \textbf{2.86E11} & \textbf{31.20} & \textbf{10.40} & \textbf{3.77E11} & {\textbf{8.68}} & {\textbf{3.32E11}} & \textbf{50.83}\\ \bottomrule
\end{tabular}}
\label{table:task}
\endgroup

\end{table}

{
\vparagraph{Accuracy and Efficiency.}
{Table~\ref{tab:deit} reports classification results for DeiT-B and ViT-L under Controlled Performance and Standard Configuration. Results for Swin and ViT-B are given in Table~\ref{table:image} in the appendix.} Across all backbones, \method reduces BOPs by over 95\%. In contrast, fixed-precision baselines such as RepQ and ERQ are less stable. RepQ can suffer severe accuracy degradation because uniform quantization amplifies representation-level errors caused by layer sensitivity and quantization noise (Theorem~\ref{thm:random_depth_error}). Although MPTQ uses mixed precision, it assumes a static inference path and ignores input-dependent layer utility under early exiting, which can lead to unnecessarily deep exits or suppressed early-exit behavior (Theorem~\ref{thm:mpq_mismatch}).}

\vparagraph{Segmentation and Detection.}
\label{apx:dense}

{To assess cross-task generalization, Table~\ref{table:task} evaluates \method on semantic segmentation and object detection using {ViT-B and DeiT-B backbones, respectively,} with ViT-EE as the exiting head. {Under Standard Configuration, MAQEE reduces segmentation BOPs by 9.49\% and improves mIoU by 4.68 percentage points over ViT-EE+MPTQ. On detection, it improves mAP@0.5 by 9.79 percentage points at comparable BOPs.} In object detection, however, early exiting with quantization still causes a noticeable accuracy drop, as detection requires multi-scale features that can be truncated by early exits and distorted by quantization.}

\vparagraph{{Joint Methods.}}
{Against joint quantization--early-exit methods (Table~\ref{tab:joint}), MAQEE reaches 88.51\% on CIFAR-100/ViT-B at $\bar{\mathrm{BOPs}}=\text{2.57E11}$, compared with 72.44\% at $\bar{\mathrm{BOPs}}=\text{3.24E11}$ for QuEE and 54.63\% for McQueen under an equivalent 4-bit budget. Relative to QuEE, MAQEE gains 16.07 percentage points while reducing BOPs by 20.7\%. It also exceeds the 87\% performance target, which neither adaptation reaches. Adaptation details are given in Appendix~\ref{apx:newexperiments}.}

\begin{samepage}
\vparagraph{Pareto Frontier.}
Figure~\ref{fig:acc_v} shows that \method forms the Pareto frontier, achieving higher accuracy at comparable cost or lower cost at comparable accuracy under both throughput and BOPs comparisons. This advantage comes from jointly optimizing bit-width allocation and exit thresholds under the APQ objective, rather than treating quantization and early exiting separately. By coupling decision-level risks (PGR and BSR) with representation-level risks (SQNR$^{-1}$ and QID), \method stabilizes exit behavior and reduces exit depth and BOPs by up to 50\% relative to strong baselines, consistent with the utilization-aware principle in Theorem~\ref{thm:mpq_mismatch}.

\par
\end{samepage}

\begingroup
\subsection{Hardware Efficiency}
Table~\ref{tab:hardware} reports ViT-B deployment measurements on an RTX 5090 at batch size 1. Compared with full-depth FP32 inference, MAQEE reduces latency from 1.608 to 0.711 ms on ImageNet and from 1.643 to 0.537 ms on CIFAR-100, yielding throughput gains of $2.26\times$ and $3.06\times$, respectively. Peak allocated memory falls by 64.9\% and 63.9\%. CIFAR-100 has a shallower average exit depth in these deployments (5.98 versus 6.65 on ImageNet), consistent with its larger runtime gain. The deployment combines packed low-bit kernels with early exiting, so the measured gains reflect both precision reduction and skipped computation. BOPs quantifies precision-weighted arithmetic, whereas GPU runtime also includes memory access, kernel execution, and exit-control overhead. Appendix~\ref{apx:hardware} provides the timing and deployment protocol.
\par
\endgroup

\subsection{{Ablation Studies}}
Table~\ref{tab:ablation} shows that the full model achieves the best controlled-performance trade-off, with stable early exits and reduced cost. {Without PGR, the allocation score omits exit prediction loss, and the model fails to reach the target accuracy.} Removing BSR impairs exit stability, as quantization perturbations more easily flip exit decisions; satisfying the same accuracy constraint therefore requires more conservative execution, increasing effective depth and BOPs. Removing SQNR$^{-1}$ eliminates relative noise-energy guidance, making mixed-precision allocation less reliable and degrading the efficiency--accuracy trade-off. Without QID, allocation omits changes in residual-update magnitude. BOPs increase even when average exit depth changes only marginally.
\par

\begingroup
\par
\begin{samepage}
\vparagraph{Representation drift.}
To examine how block quantization affects representations, we lower one block at a time from W4 to W3 and measure changes in its residual update and in the CLS token. Residual-update changes correlate with local and final-layer CLS drift (Spearman $\rho=0.779$ and $0.646$), linking local quantization sensitivity to downstream drift. Appendix~\ref{apx:riskvalidation} gives the protocol and exit-decision analysis.
\par
\end{samepage}

\endgroup

\begin{figure}[t]
\centering
\centering
\begin{minipage}[t]{0.45\textwidth}
\centering\captionsetup{font={small,color=black}}
\captionof{table}{Joint methods on CIFAR-100/ViT-B.}\label{tab:joint}

\renewcommand{\arraystretch}{1.15}
\fontsize{8}{9.5}\selectfont\setlength{\tabcolsep}{1.5pt}
\begin{tabular*}{\linewidth}{@{\extracolsep{\fill}}lrrr@{}}\toprule
Method & Top-1 (\%) & $\bar{\text{BOPs}}\!\downarrow$ & Bits\\\midrule
\textit{FP32 (full)} & \textit{90.47} & \textit{--} & \textit{32.00}\\
McQueen & 54.63 & -- & 4.00\\
QuEE & 72.44 & 3.24E11 & 4.80\\
{\textbf{MAQEE}} & {\textbf{88.51}} & {\textbf{2.57E11}} & {\textbf{4.00}}\\\bottomrule
\end{tabular*}
\end{minipage}\hfill
\begin{minipage}[t]{0.53\textwidth}
\centering\captionsetup{font={small,color=black}}
\captionof{table}{GPU efficiency on ViT-B.}\label{tab:hardware}

\renewcommand{\arraystretch}{1.15}
\fontsize{8}{9.5}\selectfont\setlength{\tabcolsep}{1.5pt}
\begin{tabular*}{\linewidth}{@{\extracolsep{\fill}}llrrrr@{}}\toprule
Dataset & Model & ms & img/s & Speedup & Peak MiB\\\midrule
\multirow{2}{*}{ImageNet} & FP32 & 1.608 & 621.9 & $1.00\times$ & 350.10\\
 & {\textbf{MAQEE}} & {\textbf{0.711}} & {\textbf{1406.3}} & {\textbf{$2.26\times$}} & {\textbf{122.81}}\\
\multirow{2}{*}{CIFAR-100} & FP32 & 1.643 & 608.5 & $1.00\times$ & 347.56\\
 & {\textbf{MAQEE}} & {\textbf{0.537}} & {\textbf{1861.6}} & {\textbf{$3.06\times$}} & {\textbf{125.47}}\\\bottomrule
\end{tabular*}
\end{minipage}
\end{figure}

\begin{figure}[t]
\centering

\begin{minipage}[b]{0.47\textwidth}

\captionsetup{font={small,color=black}}
\captionof{table}{Ablation study on {ImageNet/ViT-B}.}\label{tab:ablation}
\centering\fontsize{8}{9.5}\selectfont
\setlength{\tabcolsep}{0pt}\renewcommand{\arraystretch}{1.10}
\begin{tabular*}{\linewidth}{@{\extracolsep{\fill}}lrrrrr@{}}
\toprule
\multirow{2}{*}{Method} & \multicolumn{2}{c}{Controlled Perf.} & \multicolumn{3}{c}{Std. Config} \\ \cmidrule(l){2-6} 
 & $\hat{L}\!\downarrow$ & $\hat{\text{BOPs}}\!\downarrow$ & $\bar{L}\!\downarrow$ & $\bar{\text{BOPs}}\!\downarrow$ & Acc.{ (\%)}$\uparrow$ \\ \midrule
\textbf{MAQEE} & \textbf{7.04} & \textbf{2.85E11} & {\textbf{7.26}} & {\textbf{2.93E11}} & \textbf{79.41}\\
w/o PGR & N/A & N/A & {4.23} & {2.12E11} & 61.18 \\
w/o BSR & 7.59 & 3.12E11 & 7.01 & 2.74E11 & 76.53 \\
w/o $\text{SQNR}^{-1}$ & 7.35 & 2.98E11 & 7.28 & 2.95E11 & 78.35 \\
w/o QID & 7.17 & 2.94E11 & 7.18 & 2.95E11 & 79.06 \\ \bottomrule
\end{tabular*}\par

\end{minipage}\hfill
\begin{minipage}[b]{0.51\textwidth}
\centering
\includegraphics[width=\linewidth]{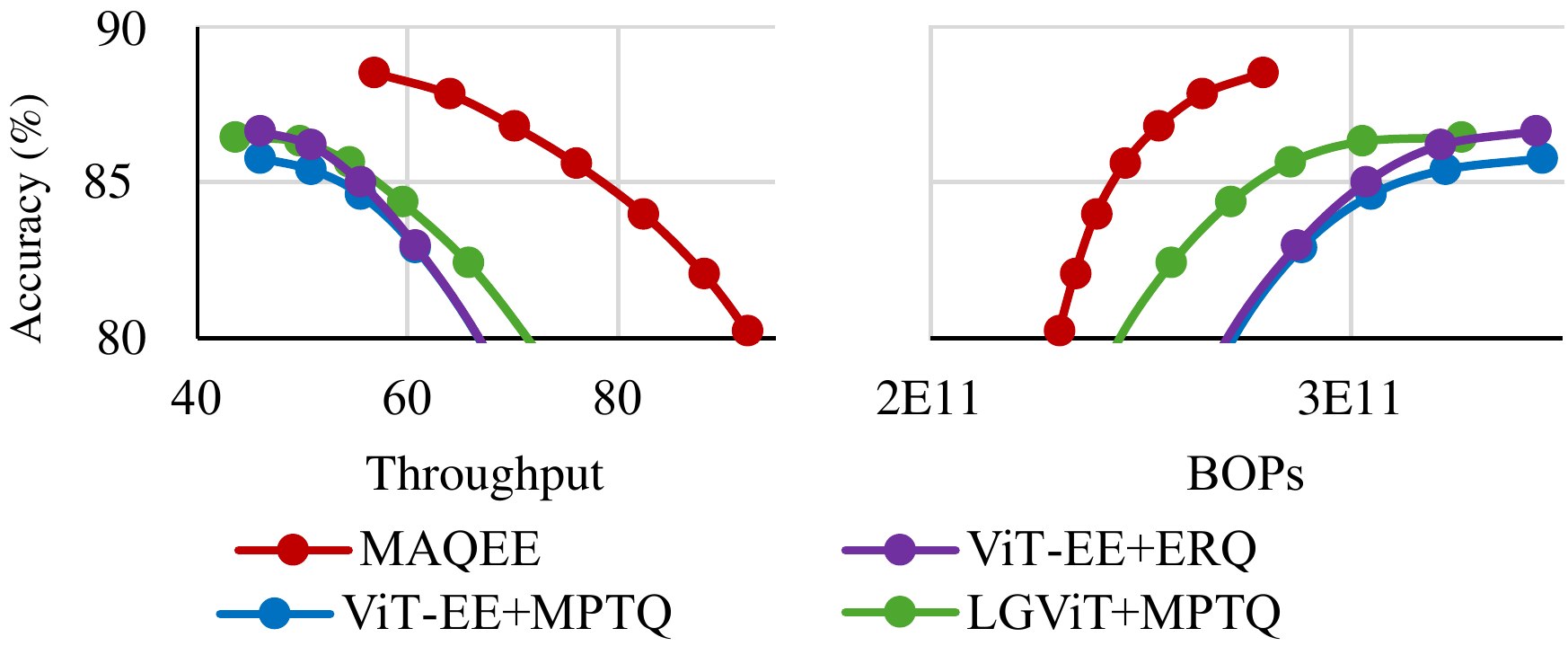}\par

\captionsetup{font={small,color=black}}
\captionof{figure}{Accuracy versus throughput and BOPs.}\label{fig:acc_v}
\end{minipage}
\end{figure}

\begingroup
\section{Limitations}
\label{sec:limitations}
MAQEE estimates layer utilization and quantization risk on a calibration distribution. Changes in input difficulty can alter both the exit distribution and the precision needed along the executed paths. Thresholds selected on the source distribution provide different operating points, but choosing thresholds for an unknown distribution shift remains an open problem. Performance on segmentation and detection depends on the quality of intermediate features, which can be affected by both early exits and quantization. Hardware gains also depend on low-bit kernel support and the overhead of dynamic execution. Extending APQ to autoregressive models requires estimating layer utilization for each token and accounting for states reused across generated tokens.
\endgroup

\section{Conclusion}
\label{sec:conclusion}

In this work, we propose \textbf{MAQEE}, a bi-level optimization framework for efficient ViT inference under the \textbf{APQ} paradigm. {APQ accounts for the expected utilization of each layer, while MAQEE combines this utilization with representation and exit-decision risks to guide precision allocation.} {Experiments} across classification, segmentation, and detection show  {improved} accuracy--{computation} trade-offs by jointly  {optimizing} precision and  {exit}  {thresholds.}  {Measurements on an RTX 5090 show GPU throughput gains of $2.26\times$ on ImageNet and $3.06\times$ on CIFAR-100.} Future work will {address adaptation under distribution shifts and} extend APQ to broader dynamic inference schemes, such as recursive Transformers, MoE routing, and layer skipping.

\clearpage
\section*{Acknowledgments}
{This work was supported by the National Science and Technology Council, Taiwan, under Grant NSTC 115-2223-E-002-001; and by the Ministry of Education, Taiwan, through the Higher Education Sprout Project---the Featured Areas Research Center Program, under Grant 115L9009.}

\newpage
\appendix

\section{Detailed Proof}
\label{sec:proof}

\subsection{Proof of Theorem~\ref{thm:margin_exit_instability}}
\begin{customthm}{\ref{thm:margin_exit_instability}}
{For an exit layer $\ell\in\mathcal E$,} assume that the confidence perturbation $\xi_{\ell,\boldsymbol{b}}(x)$ satisfies the tail bound
\begingroup
\begin{equation*}
\Pr\!\left(|\xi_{\ell,\boldsymbol{b}}(x)|>\tau_\ell\right)
\leq
q_\ell(\tau_\ell,b_\ell),
\end{equation*}
\endgroup
where $q_\ell(\tau_\ell,b_\ell)$ is non-increasing in $\tau_\ell$ and non-decreasing as $b_\ell$ decreases. Then, for any tolerance $\tau_\ell>0$, the mis-exit probability satisfies
\begingroup
\begin{equation*}
\Pr\!\left(\tilde e_\ell(x)\neq e_\ell(x)\right)
\leq
\Pr\!\left(|\Delta_\ell(x)|\leq\tau_\ell\right)
+
q_\ell(\tau_\ell,b_\ell).
\end{equation*}
\endgroup
\end{customthm}
\begin{proof}
A mis-exit can occur only when the quantization-induced perturbation crosses the full-precision confidence boundary. Therefore, if $\tilde e_\ell(x)\neq e_\ell(x)$, then $|\xi_{\ell,\boldsymbol{b}}(x)|\geq |\Delta_\ell(x)|$. For any tolerance $\tau_\ell>0$, this implies
\[
\{\tilde e_\ell(x)\neq e_\ell(x)\}
\subseteq
\{|\Delta_\ell(x)|\leq\tau_\ell\}
\cup
\{|\xi_{\ell,\boldsymbol{b}}(x)|>\tau_\ell\}.
\]
By the union bound,
\[
\Pr\!\left(\tilde e_\ell(x)\neq e_\ell(x)\right)
\leq
\Pr\!\left(|\Delta_\ell(x)|\leq\tau_\ell\right)
+
\Pr\!\left(|\xi_{\ell,\boldsymbol{b}}(x)|>\tau_\ell\right).
\]
Applying the assumed tail bound to the second term gives
\[
\Pr\!\left(\tilde e_\ell(x)\neq e_\ell(x)\right)
\leq
\Pr\!\left(|\Delta_\ell(x)|\leq\tau_\ell\right)
+
q_\ell(\tau_\ell,b_\ell).
\]
The theorem follows.
\end{proof}

\subsection{Proof of Corollary~\ref{cor:stopping_depth_perturbation}}
\label{proof:stopping_depth_perturbation}
\begin{corollary}
\label{cor:stopping_depth_perturbation}
Let $K_{\boldsymbol{\phi}}(x)$ and $\tilde K_{\boldsymbol{\phi},\boldsymbol{b}}(x)$ denote the first {exit layers in $\mathcal E$} where the full-precision and quantized confidences {reach or exceed} $\phi_\ell$, respectively, with fallback value $L$ if no earlier exit is triggered. Then
\begingroup
\begin{equation*}
\Pr\!\left(\tilde K_{\boldsymbol{\phi},\boldsymbol{b}}(x)\neq K_{\boldsymbol{\phi}}(x)\right)
\leq
\sum_{{\ell\in\mathcal E}}
\Pr\!\left(\tilde e_\ell(x)\neq e_\ell(x)\right).
\end{equation*}
\endgroup
Consequently,
\begingroup
\begin{equation*}
\left|
\mathbb{E}_{x\sim\mathcal{D}}
\!\left[\tilde K_{\boldsymbol{\phi},\boldsymbol{b}}(x)\right]
-
\mathbb{E}_{x\sim\mathcal{D}}
\!\left[K_{\boldsymbol{\phi}}(x)\right]
\right|
\leq
(L-1)
\sum_{{\ell\in\mathcal E}}
\Pr\!\left(\tilde e_\ell(x)\neq e_\ell(x)\right).
\end{equation*}
\endgroup
\end{corollary}
\begin{proof}
If the full-precision and quantized stopping depths differ, then the two exit policies must disagree at least at one {exit layer}. Otherwise, if $\tilde e_\ell(x)=e_\ell(x)$ for all {$\ell\in\mathcal E$}, the first layer satisfying the exit condition would be the same under full precision and quantization, implying $\tilde K_{\boldsymbol{\phi},\boldsymbol{b}}(x)=K_{\boldsymbol{\phi}}(x)$. Hence,
\[
\{\tilde K_{\boldsymbol{\phi},\boldsymbol{b}}(x)\neq K_{\boldsymbol{\phi}}(x)\}
\subseteq
\bigcup_{{\ell\in\mathcal E}}
\{\tilde e_\ell(x)\neq e_\ell(x)\}.
\]
Applying the union bound gives
\[
\Pr\!\left(\tilde K_{\boldsymbol{\phi},\boldsymbol{b}}(x)\neq K_{\boldsymbol{\phi}}(x)\right)
\leq
\sum_{{\ell\in\mathcal E}}
\Pr\!\left(\tilde e_\ell(x)\neq e_\ell(x)\right).
\]

For the expectation bound, since both stopping depths take values in $\{1,\ldots,L\}$, we have
\[
\left|
\tilde K_{\boldsymbol{\phi},\boldsymbol{b}}(x)
-
K_{\boldsymbol{\phi}}(x)
\right|
\leq
(L-1)\cdot
\mathbf{1}\{\tilde K_{\boldsymbol{\phi},\boldsymbol{b}}(x)\neq K_{\boldsymbol{\phi}}(x)\}.
\]
Taking expectation over $x\sim\mathcal{D}$ yields
\[
\left|
\mathbb{E}_{x\sim\mathcal{D}}[\tilde K_{\boldsymbol{\phi},\boldsymbol{b}}(x)]
-
\mathbb{E}_{x\sim\mathcal{D}}[K_{\boldsymbol{\phi}}(x)]
\right|
\leq
(L-1)\Pr\!\left(\tilde K_{\boldsymbol{\phi},\boldsymbol{b}}(x)\neq K_{\boldsymbol{\phi}}(x)\right).
\]
Combining this inequality with the union bound proves the claim. The corollary follows.
\end{proof}

\subsection{Proof of Theorem~\ref{thm:random_depth_error}}
\begin{customthm}{\ref{thm:random_depth_error}}
For any realized stopping depth $k$, the accumulated representation error satisfies
\begingroup
\begin{equation*}
\|\tilde h_{k,\boldsymbol{b}}(x)-h_k(x)\|
\leq
\sum_{\ell=1}^{k}
\varepsilon_\ell(b_\ell)
\prod_{r=\ell+1}^{k}\gamma_{r-1},
\end{equation*}
\endgroup
where the empty product is defined as $1$. Under early exiting with random stopping depth $K_{\boldsymbol{\phi}}(x)$, the expected distortion satisfies
\begingroup
\begin{equation*}
\mathbb{E}_{x\sim\mathcal{D}}
\left[
\|\tilde h_{K_{\boldsymbol{\phi}}(x),\boldsymbol{b}}(x)
-
h_{K_{\boldsymbol{\phi}}(x)}(x)\|
\right]
\leq
\mathbb{E}_{x\sim\mathcal{D}}
\left[
\sum_{\ell=1}^{K_{\boldsymbol{\phi}}(x)}
\varepsilon_\ell(b_\ell)
\prod_{r=\ell+1}^{K_{\boldsymbol{\phi}}(x)}
\gamma_{r-1}
\right].
\end{equation*}
\endgroup

\end{customthm}
\begin{proof}
Let
\[
E_k(x)=\|\tilde h_{k,\boldsymbol{b}}(x)-h_k(x)\|.
\]
By the assumed layer-wise distortion bound, for each layer $\ell=1,\ldots,L$,
\[
E_\ell(x)
\leq
\gamma_{\ell-1}E_{\ell-1}(x)+\varepsilon_\ell(b_\ell),
\]
with initial distortion $E_0(x)=0$. Unrolling this recurrence up to depth $k$ gives
\[
E_k(x)
\leq
E_0(x)\prod_{r=1}^{k}\gamma_{r-1}
+
\sum_{\ell=1}^{k}
\varepsilon_\ell(b_\ell)
\prod_{r=\ell+1}^{k}\gamma_{r-1}.
\]
Since $E_0(x)=0$, this reduces to
\[
\|\tilde h_{k,\boldsymbol{b}}(x)-h_k(x)\|
\leq
\sum_{\ell=1}^{k}
\varepsilon_\ell(b_\ell)
\prod_{r=\ell+1}^{k}\gamma_{r-1},
\]
where the empty product is defined as $1$.

Since the above bound holds for any realized stopping depth $k$, we can substitute the random stopping depth $K_{\boldsymbol{\phi}}(x)$ for $k$. Taking expectation over $x\sim\mathcal{D}$ yields
\[
\mathbb{E}_{x\sim\mathcal{D}}
\left[
\|\tilde h_{K_{\boldsymbol{\phi}}(x),\boldsymbol{b}}(x)
-
h_{K_{\boldsymbol{\phi}}(x)}(x)\|
\right]
\leq
\mathbb{E}_{x\sim\mathcal{D}}
\left[
\sum_{\ell=1}^{K_{\boldsymbol{\phi}}(x)}
\varepsilon_\ell(b_\ell)
\prod_{r=\ell+1}^{K_{\boldsymbol{\phi}}(x)}
\gamma_{r-1}
\right].
\]
The theorem follows.
\end{proof}

\subsection{Proof of Theorem~\ref{thm:mpq_mismatch}}
\begin{customthm}{\ref{thm:mpq_mismatch}}
Assume that static MPQ ranks layers by their static sensitivities $S_\ell$, while APQ ranks layers by their utilization-weighted sensitivities $u_\ell(\boldsymbol{\phi})S_\ell$. If there exist two layers $i$ and $j$ such that
\begingroup\begin{equation*}
S_i>S_j
\quad\text{but}\quad
u_i(\boldsymbol{\phi})S_i<u_j(\boldsymbol{\phi})S_j,
\end{equation*}\endgroup
then the static MPQ ranking is inconsistent with the APQ ranking. Moreover, under equal marginal bit-cost increments and marginal APQ benefit proportional to $u_\ell(\boldsymbol{\phi})S_\ell$, the static MPQ allocation can be suboptimal for the APQ objective.
\end{customthm}
\begin{proof}
The first claim follows directly from the assumed sensitivity and utilization-weighted orderings. Static MPQ ranks layer $i$ above layer $j$ because $S_i>S_j$, whereas APQ ranks layer $j$ above layer $i$ because $u_i(\boldsymbol{\phi})S_i<u_j(\boldsymbol{\phi})S_j$. Hence, the two rankings are inconsistent.

For the second claim, consider a fixed bit budget that allows assigning one additional bit-width increment to exactly one of the two layers. Under equal marginal bit-cost increments, the two choices have the same budget cost. If the marginal APQ benefit is proportional to $u_\ell(\boldsymbol{\phi})S_\ell$, then assigning the additional precision to layer $j$ yields a larger APQ benefit than assigning it to layer $i$, since
\[
u_j(\boldsymbol{\phi})S_j
>
u_i(\boldsymbol{\phi})S_i .
\]
However, static MPQ selects layer $i$ because $S_i>S_j$. Therefore, the static MPQ allocation is suboptimal for the APQ objective under this budget. The theorem follows.
\end{proof}

\subsection{Proof of Theorem~\ref{thm:np}}
\label{proof:np}
\begin{theorem}\label{thm:np}
The APQ bi-level optimization is NP-hard. Moreover, each level is NP-hard on its own:  
(i) with fixed exit thresholds $\boldsymbol{\phi}$, the inner bit-width allocation is NP-hard; and  
(ii) with fixed bit-widths $\boldsymbol{b}$, the outer exit-threshold optimization is also NP-hard.
\end{theorem}
\begin{proof}
    
We first prove that the inner-level bit-width allocation is NP-hard when the thresholds $\boldsymbol\phi$ are fixed (Lemma~\ref{lem:inner-nphard}), 
and then show that the outer-level threshold optimization is NP-hard when the bit-widths $\boldsymbol b$ are fixed (Lemma~\ref{lem:outer-nphard}). 

\begin{lemma}\label{lem:inner-nphard}
Fix any thresholds $\boldsymbol{\phi}$. Consider the decision problem: 
``Does there exist a bit allocation $\boldsymbol b\in\mathcal{B}^L$ (e.g., discrete per-layer bit choices) such that the expected complexity $\mathbb{E}_x[\sum_{\ell} u_\ell(x,\boldsymbol\phi) c_\ell(b_\ell)]$ is at most $C_{\max}$ while the expected loss $\mathbb{E}_{(x,y)}[\mathcal{L}_{\rm CE}(f_{\theta,\boldsymbol b,\boldsymbol\phi}(x),y)]$ is at most $\mathcal L_{\max}$?'' 
This problem is NP-hard.
\end{lemma}

\begin{proof}
Reduce from the classical \emph{0--1 Knapsack}. Given items $\{1,\dots,L\}$ with values $v_\ell>0$, weights $w_\ell>0$, and capacity $W$, ask whether there exists a subset $S$ with $\sum_{\ell\in S} w_\ell\le W$ and $\sum_{\ell\in S} v_\ell \ge V$. 
Construct an APQ instance with $L$ layers and a \emph{binary} bit-choice set $\mathcal{B}_\ell=\{b_\ell^{(0)},b_\ell^{(1)}\}$ per layer:
i) Complexity mapping: set $c_\ell(\cdot)$ and (fixed) utilization $u_\ell(\cdot,\boldsymbol\phi)$ so that 
$\mathbb{E}_x[u_\ell c_\ell(b_\ell^{(1)})]-\mathbb{E}_x[u_\ell c_\ell(b_\ell^{(0)})]=w_\ell$.
ii)  Accuracy mapping: choose a dataset and a head such that turning on the high-bit option at layer $\ell$ reduces the loss by exactly $v_\ell$ in expectation (e.g., via a separable surrogate $D_\ell(b_\ell)$ scaled appropriately).
Set $C_{\max}=\sum_{\ell}\mathbb{E}_x[u_\ell c_\ell(b_\ell^{(0)})]+W$ and $\mathcal L_{\max}=L_0-V$, where $L_0$ is the loss under all baselines $b_\ell^{(0)}$. 
Then selecting $S=\{\ell: b_\ell=b_\ell^{(1)}\}$ is feasible iff the knapsack instance is feasible. Hence NP-hardness follows.
\end{proof}

\begin{lemma}\label{lem:outer-nphard}
Fix any bit-widths $\boldsymbol b$. Consider the decision problem: 
``Does there exist thresholds $\boldsymbol\phi$ such that the expected complexity $C_{\mathrm{APQ}}(\boldsymbol b,\boldsymbol\phi)$ is at most $C_{\max}$ while the expected loss $\mathbb{E}_{(x,y)}[\mathcal{L}_{\rm CE}(f_{\theta,\boldsymbol b,\boldsymbol\phi}(x),y)]$ is at most $\mathcal L_{\max}$?'' 
This problem is NP-hard.
\end{lemma}

\begin{proof}
Reduce from \emph{0--1 Knapsack}. Take a two-exit cascade (layers $1$ and $2$) with fixed logits/softmax scores precomputed on a finite dataset $\{x_i\}_{i=1}^n$. Let $\tilde\rho_{1,\boldsymbol b}(x_i)\in(0,1)$ be the confidence at the first exit and suppose the final exit (layer~2) always predicts correctly with unit cost, while exiting at layer~1 costs zero additional depth but may misclassify some points. 
Associate each item $i$ with a sample $x_i$ and set:
i) If $x_i$ exits at layer~1 (i.e., $\tilde\rho_{1,\boldsymbol b}(x_i)\ge \phi_1$), we \emph{save} expected depth $w_i>0$ (benefit), but we incur a loss penalty $p_i\ge 0$ if that early decision is wrong. ii) If it proceeds to layer~2 (i.e., $\tilde\rho_{1,\boldsymbol b}(x_i)<\phi_1$), we pay depth $w_i$ but incur no penalty. 

Choose the surrogate loss so that total early-exit penalty equals $\sum_{i\in S} p_i$, where $S=\{i: \tilde\rho_{1,\boldsymbol b}(x_i)\ge \phi_1\}$. Then the constraints 
$C_{\mathrm{APQ}}(\boldsymbol b,\boldsymbol\phi)\le C_{\max}$ and $\mathbb{E}[\mathcal{L}_{\rm CE}]\le\mathcal L_{\max}$ become 
$\sum_{i\in S} w_i \ge V$ and $\sum_{i\in S} p_i \le W$, exactly the knapsack feasibility test. 
Thus finding $\phi_1$ is NP-hard. The argument extends to multiple exits by letting only the first threshold be active.
\end{proof}

Concluding Lemmas~\ref{lem:inner-nphard} and~\ref{lem:outer-nphard}, each level already involves an NP-hard subproblem, even when the other level is fixed. Consequently, the joint bi-level problem is at least as hard as solving either level in isolation, thereby implying NP-hardness. Hence, the theorem follows.
\end{proof}

\section{Experiments}
\label{apx:exp}

\subsection{Implementation Details}
\label{apx:implement}

All training runs are conducted on 
NVIDIA H100 GPUs. We adopt the SegFormer~\citep{xie2021segformer} head on a ViT-B for semantic segmentation, and conduct object detection experiments with a {DeiT-B backbone}~\citep{fang2021you}. For each model, we configure four exit heads: the first two are convolutional, and the latter two are attention-based. The exit heads are inserted at layers 4, 6, 8, and 10 in DeiT-B and ViT-B, and at layers 8, 12, 16, and 20 in ViT-L. For the Swin model, they are placed at the 2\textsuperscript{nd}, 4\textsuperscript{th}, 14\textsuperscript{th}, and 20\textsuperscript{th} layers, approximately uniformly distributed with respect to computational cost.  All models are trained with AdamW and a cosine-decay schedule with a $1\times 10^{-5}$ peak LR.

{The admissible weight precisions range from $b_{\min}=3$ to $b_{\max}=6$.} {The empirical solver settings and recovery losses are specified in Appendices~\ref{apx:solver} and~\ref{apx:recovery}.}

\subsection{Results on {DeiT-B} and ViT-L}
\label{apx:deit}

{On DeiT-B, MAQEE exits approximately 3--4 layers earlier than the strongest quantized baseline at the same target accuracy (Table~\ref{tab:deit}).} Although MPQ alleviates part of the accuracy degradation of FPQ under Early Exiting, its assumption of a fixed inference path constrains exit efficiency, leaving it less effective than MAQEE. A similar pattern is observed for ViT-L. While the increased parameter count provides higher capacity, the greater depth also amplifies quantization-induced perturbations, leading to limited gains in quantized model performance. {On ViT-L, MAQEE exits approximately 9--11 layers earlier than the strongest quantized baseline at the same target accuracy.} 

\begingroup
\begin{table}[t]

\captionsetup{font={small,color=black}}
\centering\small
\caption{{Classification results on Swin and ViT-B.}}

\resizebox{\columnwidth}{!}{\begin{tabular}{@{}cccccccccccccc@{}}
\toprule
\multirow{3}{*}{} & \multirow{3}{*}{\begin{tabular}[c]{@{}c@{}}Bits\\ (W/A)\end{tabular}} & \multirow{3}{*}{Quant} & \multirow{3}{*}{EE} & \multicolumn{5}{c}{CIFAR-100} & \multicolumn{5}{c}{ImageNet} \\ \cmidrule(l){5-14} 
 &  &  &  & \multicolumn{2}{c}{Controlled Perf.} & \multicolumn{3}{c}{Std. Config} & \multicolumn{2}{c}{Controlled Perf.} & \multicolumn{3}{c}{Std. Config} \\ \cmidrule(l){5-14} 
 &  &  &  & $\hat{L}\downarrow$ & $\hat{\text{BOPs}}\downarrow$ & $\bar{L}\downarrow$ & $\bar{\text{BOPs}} \downarrow$ & Acc.{ (\%)}$\uparrow$ & $\hat{L}\downarrow$ & $\hat{\text{BOPs}}\downarrow$ & $\bar{L}\downarrow$ & $\bar{\text{BOPs}} \downarrow$ & Acc.{ (\%)}$\uparrow$ \\ \midrule
\multirow{10}{*}{\rotatebox[origin=c]{90}{Swin}} & \multirow{3}{*}{FP32} & \multirow{3}{*}{{--}} & {\textit{--}} & {\textit{--}} & {\textit{--}} & {\textit{24.00}} & \textit{1.39E13} & \textit{94.16} & {\textit{--}} & {\textit{--}} & {\textit{24.00}} & \textit{1.39E13} & \textit{83.50} \\ \cmidrule(l){4-14} 
 &  &  & ViT-EE & 13.59 & 9.30E12 & 16.47 & 1.09E13 & 90.80 & 16.62 & 1.23E13 & 18.23 & 1.35E13 & 83.68 \\
 &  &  & LGViT & 9.88 & 1.02E13 & 13.80 & 9.41E12 & 89.38 & 11.13 & 8.83E12 & 14.20 & 1.10E13 & 82.70 \\ \cmidrule(l){2-14} 
 & \multirow{4}{*}{FP4/4} & \multirow{2}{*}{RepQ} & ViT-EE & N/A & N/A & 18.72 & 2.12E11 & 82.23 & N/A & N/A & 20.16 & 2.27E11 & 74.75 \\
 &  &  & LGViT & N/A & N/A & 13.69 & 1.65E11 & 74.38 & N/A & N/A & 18.97 & 2.19E11 & 69.85 \\ \cmidrule(l){3-14} 
 &  & \multirow{2}{*}{ERQ} & ViT-EE & 23.72 & 2.69E11 & 21.26 & 2.29E11 & 83.58 & 23.46 & 2.65E11 & 20.96 & 2.31E11 & 75.45 \\
 &  &  & LGViT & N/A & N/A & 19.76 & 1.96E11 & 76.74 & N/A & N/A & 20.80 & 2.38E11 & 73.97 \\ \cmidrule(l){2-14} 
 & \multirow{3}{*}{MP4/4} & \multirow{2}{*}{MPTQ} & ViT-EE & N/A & N/A & 21.01 & 2.36E11 & 81.62 & N/A & N/A & 20.99 & 2.39E11 & 71.55 \\
 &  &  & LGViT & N/A & N/A & 16.56 & 1.90E11 & 75.07 & N/A & N/A & 21.23 & 2.42E11 & 74.35 \\ \cmidrule(l){3-14} 
 &  & \multicolumn{2}{c}{\textbf{MAQEE}} & \textbf{12.58} & \textbf{1.53E11} & \textbf{13.79} & \textbf{1.65E11} & \textbf{89.13} & \textbf{17.60} & \textbf{2.09E11} & \textbf{18.75} & \textbf{2.12E11} & \textbf{80.78}\\ \midrule
\multirow{10}{*}{\rotatebox[origin=c]{90}{ViT-B}} & \multirow{3}{*}{FP32} & \multirow{3}{*}{{--}} & {\textit{--}} & {\textit{--}} & {\textit{--}} & {\textit{12.00}} & \textit{1.80E13} & \textit{90.47} & {\textit{--}} & {\textit{--}} & {\textit{12.00}} & \textit{1.80E13} & \textit{83.10} \\ \cmidrule(l){4-14} 
 &  &  & ViT-EE & 7.09 & 1.81E13 & 7.27 & 1.89E13 & 87.51 & 7.65 & 1.90E13 & 8.68 & 2.05E13 & 79.60 \\
 &  &  & LGViT & 5.90 & 1.68E13 & 6.33 & 1.76E13 & 88.51 & 6.69 & 1.80E13 & 7.08 & 1.86E13 & 80.30 \\ \cmidrule(l){2-14} 
 & \multirow{4}{*}{FP4/4} & \multirow{2}{*}{RepQ} & ViT-EE & 11.78 & 3.94E11 & 9.90 & 3.51E11 & 85.71 & N/A & N/A & 10.83 & 3.71E11 & 74.67 \\
 &  &  & LGViT & N/A & N/A & 8.08 & 3.13E11 & 73.52 & N/A & N/A & 8.74 & 3.31E11 & 72.30 \\ \cmidrule(l){3-14} 
 &  & \multirow{2}{*}{ERQ} & ViT-EE & 8.77 & 3.23E11 & 9.28 & 3.35E11 & 87.82 & N/A & N/A & 9.88 & 3.49E11 & 77.35 \\
 &  &  & LGViT & N/A & N/A & 9.00 & 3.34E11 & 77.64 & N/A & N/A & 9.12 & 3.40E11 & 73.06 \\ \cmidrule(l){2-14} 
 & \multirow{3}{*}{MP4/4} & \multirow{2}{*}{MPTQ} & ViT-EE & 8.93 & 3.27E11 & 8.71 & 3.22E11 & 85.41 & N/A & N/A & 10.26 & 3.58E11 & 76.56 \\
 &  &  & LGViT & 10.80 & 3.74E11 & 8.05 & 3.14E11 & 84.44 & 10.25 & 3.68E11 & 9.33 & 3.45E11 & 77.53 \\ \cmidrule(l){3-14} 
 &  & \multicolumn{2}{c}{\textbf{MAQEE}} & {\textbf{5.04}} & {\textbf{2.37E11}} & {\textbf{5.83}} & {\textbf{2.57E11}} & {\textbf{88.51}} & \textbf{7.04} & \textbf{2.85E11} & \textbf{7.26} & \textbf{2.93E11} & \textbf{79.41}\\ \bottomrule
\end{tabular}}

\label{table:image}
\end{table}

\subsection{Additional Experimental Settings}
\label{apx:newexperiments}
McQueen and QuEE are implemented from their papers on CIFAR-100/ViT-B. McQueen optimizes continuous layer precision under an adaptive budget penalty and then discretizes the allocation. QuEE retains W3, W4, and W5 block parameterizations to construct candidate paths. Its original higher-bit ViT setting and the W3--W5 adaptation therefore differ in the quality of the candidate classifiers. MAQEE uses a single mixed-precision allocation. Long-Tail gives shallow layers higher precision, while utilization-only allocates precision according to execution frequency; both share the average W4 and A4 constraints in Table~\ref{tab:simple}.

\subsection{Risk Diagnostics}
\label{apx:riskvalidation}
The perturbation study starts from uniform W4A4 and lowers one block to W3 at a time. The 10,000 evaluation images are disjoint from the risk-score set, and the reference model is restored before each intervention. Perturbing blocks in the highest quartile of the residual-update diagnostic score yields 2.47 times the local CLS drift and 1.63 times the final-layer drift observed in the lowest quartile. Controlling for layer depth gives a partial rank correlation of 0.671 ($p=0.0266$) with local drift.

\begin{samepage}
For BSR, we examine downstream exit gates reached by both the reference and perturbed models. Decisions within 0.02 of the confidence threshold flip at a rate of 48.62\%, compared with 26.23\% for other decisions. The risk ratio is 1.85 (95\% bootstrap CI: [1.76, 1.94]). Near-boundary decisions have premature- and delayed-exit risk ratios of 2.70 and 1.65, respectively. Flip rates decrease from 49.32\% at boundary distances below 0.005 to 16.60\% at distances of at least 0.20. These results connect proximity to the exit threshold with observed decision instability.
\par
\end{samepage}

\subsection{Deployment Protocol}
\label{apx:hardware}
The deployment uses PyTorch 2.9.0+cu128 on an NVIDIA GeForce RTX 5090. CIFAR-100 uses the ordered 10,000-image test set; ImageNet uses 10,000 shuffled validation images with seed 777. Inputs are resized to $224\times224$. The runtime divides the backbone into blocks 1--4, 5--6, 7--8, 9--10, and 11--12. Each segment includes confidence evaluation. The host reads the resulting exit decision before executing the next segment. Logical W2--W4 weights use packed INT4 containers, while W5--W8 use INT8 containers; activations use A4. Peak allocation measures live allocated memory, separately from reserved CUDA Graph pools. The deployment configurations are evaluated with their own fixed thresholds.

\subsection{Allocation Strategies}
\label{apx:allocation}
\begin{table}[t]
\centering\small\captionsetup{font={small,color=black}}
\caption{Allocation strategies on ImageNet/ViT-B with an average W4 budget and A4 activations.}
\label{tab:simple}
\begin{tabular}{lrrr}\toprule
Allocation & Top-1 (\%) & $\bar{L}\!\downarrow$ & $\bar{\text{BOPs}}\!\downarrow$\\\midrule
Uniform & 73.56 & 9.12 & 3.40E11\\
Long-Tail & 78.63 & 8.75 & 3.36E11\\
Utilization-only & 78.72 & 9.26 & 3.29E11\\
{\textbf{MAQEE}} & {\textbf{79.41}} & {\textbf{7.26}} & {\textbf{2.93E11}}\\\bottomrule
\end{tabular}
\end{table}
Table~\ref{tab:simple} separates the benefit of utilization from that of risk-aware allocation. Long-Tail assigns more precision to shallow layers, while utilization-only uses their execution frequency. MAQEE improves Top-1 accuracy over utilization-only by 0.69 percentage points and reduces computation by 3.60E10 BOPs. The full optimization evaluates the effect of precision changes on representations and exit decisions, adding calibration work to the simpler allocation rules. At inference, it uses the resulting fixed allocation and confidence thresholds.

\subsection{Threshold Policies and ImageNet-R}
\label{apx:ood}
We hold the ImageNet/ViT-B allocation fixed and vary the source-domain exit policy. Table~\ref{tab:threshold_policies} shows that the aggressive, default, and conservative policies span 75.57--80.10\% ImageNet accuracy and 2.61E11--3.10E11 BOPs.
\begin{table}[t]\centering\small
\caption{ImageNet/ViT-B performance under different exit policies.}\label{tab:threshold_policies}
\begin{tabular}{lrrr}\toprule
Policy & ImageNet Top-1 (\%) & $\bar{L}\!\downarrow$ & $\bar{\text{BOPs}}\!\downarrow$\\\midrule
Aggressive & 75.57 & 5.68 & 2.61E11\\
Default & 79.41 & 7.26 & 2.93E11\\
Conservative & 80.10 & 8.13 & 3.10E11\\\bottomrule
\end{tabular}\end{table}
All three policies were constructed on ImageNet before the ImageNet-R experiments. We evaluate the default and conservative policies on ImageNet-R with their thresholds, allocation, and model parameters fixed. Table~\ref{tab:imagenet_r} reports the results on ImageNet-R as a distribution-shift stress test.
\begin{table}[t]\centering\small
\caption{ImageNet-R performance and GPU efficiency on ViT-B.}\label{tab:imagenet_r}
\begin{tabular}{lrrrr}\toprule
Model / policy & Top-1 (\%) & $\bar{L}\!\downarrow$ & ms & img/s\\\midrule
FP32, full depth & 51.55 & 12.00 & 1.608 & 622.0\\
Uniform W4A4, full depth & 38.34 & 12.00 & 1.279 & 781.6\\
Uniform W4A4 + EE, default & 15.16 & 8.57 & -- & --\\
Uniform W4A4 + EE, conservative & 31.07 & 10.85 & -- & --\\
{\textbf{MAQEE, default}} & {\textbf{16.35}} & {\textbf{8.54}} & {\textbf{0.940}} & {\textbf{1063.8}}\\
{\textbf{MAQEE, conservative}} & {\textbf{35.15}} & {\textbf{10.41}} & {\textbf{1.161}} & {\textbf{861.4}}\\\bottomrule
\end{tabular}\end{table}
The uniform and MAQEE comparisons share the backbone, exit heads, ImageNet calibration data, and source-domain threshold-selection procedure. Both allocations experience a substantial accuracy decrease with the default policy. MAQEE improves accuracy over uniform W4A4 + EE by 1.19 and 4.08 percentage points under the default and conservative policies. Runtime measurements cover all 30,000 ImageNet-R images. Relative to full-depth W4A4, the conservative MAQEE policy reduces latency by 9.26\% while trading 3.19 percentage points of Top-1 accuracy.

\subsection{Transfer to a CNN Backbone}
\label{apx:cnn}
We evaluate ConvNeXt-Base on ImageNet with EERO~\citep{valade2024eero} early-exit heads and QDrop~\citep{wei2022qdrop} quantization. MAQEE assigns W3--W5 to 36 pointwise blocks with an exact average weight precision of W4 and fixed A4 activations. The quantized models share exit heads, quantization technique, and evaluation protocol.
\begin{table}[t]\centering\small
\caption{ImageNet classification results on ConvNeXt-Base.}\label{tab:cnn}
\begin{tabular}{lrrr}\toprule
Model & Top-1 (\%) & $\bar{L}\!\downarrow$ & $\bar{\text{BOPs}}\!\downarrow$\\\midrule
FP32, full depth & 83.99 & 36.00 & --\\
EERO, uniform W4A4 & 72.92 & 30.51 & 2.21E11\\
{\textbf{MAQEE, mixed precision}} & {\textbf{77.55}} & {\textbf{25.34}} & {\textbf{1.90E11}}\\\bottomrule
\end{tabular}\end{table}
In Table~\ref{tab:cnn}, MAQEE improves accuracy by 4.63 percentage points and reduces computation by 3.10E10 BOPs. The allocation varies within stages, assigning additional precision to selected blocks rather than following a monotonic depth schedule.

\subsection{Risk and Allocation Details}
\label{apx:solver}
The threshold solver ranks policies by calibration cost and exit risk subject to the accuracy target in Eq.~\eqref{eq:outer-opt}. The allocator maximizes the utility in Eq.~\eqref{eq:psi_ell_bits} under the bit budget $B=L\bar b$.
For classification, the implementation uses mean exit cross-entropy to form the allocation scores. Subtracting the shared mean final-classifier loss gives the PGR in Eq.~\eqref{eq:pgr}; this common offset cancels under the percentile normalization below. Each exit's loss and BSR are assigned to the blocks since the preceding exit; the final segment uses the final-classifier loss and zero BSR. For each exit, $u_e(x,\boldsymbol\phi)$ records whether the current quantized model reaches it. BSR counts samples that reach an exit and satisfy $|\tilde\rho_e-\phi_e|<0.02$, divided by the calibration-set size.

For QID, $\sigma_\ell$ and $\tilde{\sigma}_\ell$ are the teacher and student residual-update RMS magnitudes, each averaged over batches. We then take their student-to-teacher ratio with $\epsilon=10^{-8}$ in Eq.~\eqref{eq:qid}. Near a ratio of one, squared log deviation approximates squared relative deviation. For SQNR, let $d_\ell=h_\ell-h_{\ell-1}$ and $\tilde d_\ell=\tilde h_\ell-\tilde h_{\ell-1}$ denote their residual updates. We use
\begin{equation}
 r_{\mathrm{SQNR},\ell}=10\log_{10}
 \frac{\sum_{x\in\mathcal D_{\mathrm{cal}}}\|d_\ell(x)-\tilde d_\ell(x)\|^2}
 {\sum_{x\in\mathcal D_{\mathrm{cal}}}\|d_\ell(x)\|^2}.
\end{equation}
This residual-update estimator measures layer-local quantization noise for allocation. Components are scaled and clipped using their 10th and 90th percentiles across layers; constant components map to 0.5. Weights for SQNR, PGR, BSR, and QID are 1.0, 0.5, 0.5, and 0.5. The allocation sensitivity is $s_\ell=s_{\min}+(1-s_{\min})\mathcal N(R)_\ell$, with $s_{\min}=0.05$.

\begin{samepage}
For allocation, we model the remaining quantization risk at bit-width $b$ as $s_\ell 2^{-2b}$, holding $s_\ell$ fixed during each allocation step. The estimated reduction from one additional bit and its accumulated benefit are
\begin{equation}\label{eq:allocation_benefit}
\begin{aligned}
\Delta R_\ell(b)&=s_\ell\bigl(2^{-2b}-2^{-2(b+1)}\bigr),\\
U_\ell(b)&=\sum_{q=b_{\min}}^{b-1}\Delta R_\ell(q)
=s_\ell\bigl(2^{-2b_{\min}}-2^{-2b}\bigr).
\end{aligned}
\end{equation}
\end{samepage}
The factor $2^{-2b}$ models squared quantization error as proportional to the squared quantization step. Candidate BOPs are evaluated using the current exit distribution. For each layer and candidate bit-width, the solver computes the cost increase relative to the all-minimum-bit allocation; this equals the sum of the one-bit increments defined in Section~\ref{sec:iterative}.

\begin{samepage}
Weight bit-widths range from 3 to 6, with activations fixed at 4 bits. For 12 blocks at average W4, $B=48$. Dynamic programming tracks accumulated utility and BOPs at each bit sum. At each layer, it retains the path with the highest utility, choosing lower BOPs when utilities are equal. 

Threshold search uses confidence quantiles and local refinement. Among policies that meet the accuracy target, we compare BOPs first, followed by the two components of $R_{\mathrm{exit}}$ in Eq.~\eqref{eq:outer-opt}: BSR and exit prediction loss. For each sample and intermediate exit, we record whether the sample reaches that exit and its confidence lies within 0.02 of the threshold. BSR is the mean of these indicators over all calibration samples and intermediate exits. Exit prediction loss is the mean loss at the selected exit for each sample, including those that reach the final classifier. Any remaining ties favor the policy whose accuracy exceeds the target by the smaller amount. If no policy meets the target, it selects the highest-accuracy policy, using BOPs to break ties. Across rounds, selection minimizes BOPs among target-feasible models; otherwise, it chooses the model whose accuracy is closest to the target, using BOPs to break ties.
\par
\end{samepage}

\subsection{Recovery Implementation}
\label{apx:recovery}
For classification, let $z_{\mathrm T}(x)$ be the frozen teacher's final logits and $\tilde z_e(x)$ the student logits at classifier $e$. We use
\begin{equation}
\begin{aligned}
p_{\mathrm T}(x)&=\operatorname{softmax}(z_{\mathrm T}(x)),\quad
p_e(x)=\operatorname{softmax}(\tilde z_e(x)),\\
\ell_{\mathrm{CE},e}(x,y)&=\operatorname{CE}(\tilde z_e(x),y),\quad
\ell_{\mathrm{KL},e}(x)=D_{\mathrm{KL}}(p_{\mathrm T}(x)\|p_e(x)).
\end{aligned}
\end{equation}
Both distributions use temperature 1. The losses $\mathcal L_{\mathrm{CE},e}$ and $\mathcal L_{\mathrm{KL},e}$ in Eq.~\eqref{eq:recovery_loss} average these sample losses, using the mask below at intermediate exits. If $\beta$ is the intermediate-exit loss coefficient and $v_e$ the relative exit weight, then $\omega_e=\beta v_e/\sum_{j\in\mathcal E}v_j$ for intermediate exits and $\omega_L=1-\beta$.

\paragraph{Confidence-based sample masking.}
Let $e_1<\cdots<e_H$ denote the intermediate exits. During recovery, the active-sample mask at exit $e_j$ is
\begin{equation}
 m_j(x)=\prod_{k=1}^{j}\mathbf 1\{\tilde\rho_{e_k}(x)<1.1\phi_{e_k}\}.
\end{equation}
CE and KL are averaged over active samples, with a minimum denominator of one. Crossing $1.1\phi_e$ excludes a sample at that and all later intermediate exits; the final loss is unmasked. Inference exits at the first confidence reaching $\phi_e$.

\endgroup

\clearpage
\section{Bi-Level Optimization Algorithm}\label{apx:code}

\label{algo_code}
\begin{algorithm}[h]

\caption{Bi-level optimization with recovery}
\label{alg:maqee}
\begin{algorithmic}
\Require FP anchor $\theta_{\mathrm{FP}}$, frozen teacher $\theta_{\mathrm T}$; datasets $\mathcal D_{\mathrm{train}},\mathcal D_{\mathrm{cal}},\mathcal D_{\mathrm{eval}}$; initial allocation $\boldsymbol b^{(0)}$, thresholds $\boldsymbol\phi^{(0)}$; target $A_{\mathrm{target}}$, average precision $\bar b$, admissible bits $\mathcal B$, round budget $K_{\max}$.
\Ensure Recovered model $\theta^\star$, allocation $\boldsymbol b^\star$, thresholds $\boldsymbol\phi^\star$.
\State \textbf{Initialization:} $B\leftarrow L\bar b$, $\mathcal V\leftarrow\varnothing$, $\mathcal C\leftarrow\varnothing$.
\For{$k=0,\ldots,K_{\max}-1$}
 \State $\mathcal V\leftarrow\mathcal V\cup\{\boldsymbol b^{(k)}\}$.
 \State \textbf{Self-distillation recovery:}
 \State Construct the quantized model with $\boldsymbol b^{(k)}$ from $\theta_{\mathrm{FP}}$.
 \State $\theta^{(k)}\leftarrow\operatorname{Recover}(\mathcal D_{\mathrm{train}},\theta_{\mathrm T},\boldsymbol\phi^{(k)};\mathcal L_{\mathrm{rec}})$ (Eq.~\eqref{eq:recovery_loss}).
 \State Collect fresh confidences $\tilde\rho_\ell(x)$ and representation statistics on $\mathcal D_{\mathrm{cal}}$.
 \State \textbf{Outer loop (threshold search):}
 \State Initialize threshold search with $\boldsymbol\phi^{(k)}$; build and refine candidate policies $\mathcal T_\phi$ from confidence quantiles.
 \State $\mathcal F_k\leftarrow\{\boldsymbol\phi\in\mathcal T_\phi:\widehat{\mathrm{Acc}}(\boldsymbol b^{(k)},\boldsymbol\phi)\ge A_{\mathrm{target}}\}$.
 \State $\widehat{\boldsymbol\phi}^{(k)}\leftarrow\operatorname*{arg\,min}^{\mathrm{lex}}_{\boldsymbol\phi\in\mathcal F_k}\bigl(C_{\mathrm{APQ}}(\boldsymbol b^{(k)},\boldsymbol\phi),R_{\mathrm{exit}}(\boldsymbol b^{(k)},\boldsymbol\phi)\bigr)$ (Eq.~\eqref{eq:outer-opt}).
 \State Evaluate accuracy $a_k$ and BOPs $c_k$ on $\mathcal D_{\mathrm{eval}}$.
 \State $\mathcal C\leftarrow\mathcal C\cup\{(\theta^{(k)},\boldsymbol b^{(k)},\widehat{\boldsymbol\phi}^{(k)},a_k,c_k)\}$.
 \State \textbf{Inner loop (bit allocation):}
 \State Compute block risks $r_{j,\ell}$ under $\widehat{\boldsymbol\phi}^{(k)}$, mapping exit risks to block segments.
 \State $R_\ell\leftarrow\sum_j w_j\mathcal N(r_j)_\ell$ for all $\ell=1,\ldots,L$.
 \For{$\ell=1,\ldots,L$}
  \State $s_\ell\leftarrow s_{\min}+(1-s_{\min})\mathcal N(R)_\ell$.
  \State $U_\ell(b)\leftarrow\sum_{q=b_{\min}}^{b-1}\Delta R_\ell(q)$, $b\in\mathcal B$ (Eq.~\eqref{eq:allocation_benefit}).
  \State Compute candidate BOPs increments under $\widehat{\boldsymbol\phi}^{(k)}$.
 \EndFor
 \State $\boldsymbol b^{(k+1)}\leftarrow\arg\max_{\boldsymbol b\in\mathcal B^L,\,\sum_\ell b_\ell=B}\sum_{\ell=1}^{L}U_\ell(b_\ell)$ by dynamic programming (Eq.~\eqref{eq:psi_ell_bits}).
 \State At each bit-sum state, prefer lower accumulated BOPs when utilities are equal.
 \If{$\boldsymbol b^{(k+1)}\in\mathcal V$}
  \State \textbf{break}
 \EndIf
 \State $\boldsymbol\phi^{(k+1)}\leftarrow\widehat{\boldsymbol\phi}^{(k)}$.
\EndFor
\State $\mathcal C_{\mathrm{feas}}\leftarrow\{c\in\mathcal C:a(c)\ge A_{\mathrm{target}}\}$;
       $c^\star\leftarrow\arg\min_{c\in\mathcal C_{\mathrm{feas}}}\mathrm{BOPs}(c)$.
\State \Return $(\theta^\star,\boldsymbol b^\star,\boldsymbol\phi^\star)$ from $c^\star$.

\end{algorithmic}
\end{algorithm}

\newpage

\clearpage
\section{Notation Table}
\label{apx:notation}

\begin{table}[h]

\centering
\small
\renewcommand{\arraystretch}{0.96}

\caption{Notation used throughout the paper.}
\label{tab:notation}
\begin{tabular}{ll}
\toprule
\textbf{Symbol} & \textbf{Description} \\
\midrule
\(L\) & Number of layers. \\
\(\ell\) & Layer index, \(\ell=1,\ldots,L\). \\
\(x,y\) & Input sample and label. \\
\(\mathcal D\) & Data distribution. \\
\(\theta\) & Model parameters. \\
\(\mathcal{B}\) & Set of admissible bit-widths, e.g., \(\{2,4,8\}\). \\
\(b_\ell\) & Bit-width assigned to layer \(\ell\), with \(b_\ell\in\mathcal{B}\). \\
\(\boldsymbol b\) & Layer-wise bit allocation, \(\boldsymbol b=\{b_\ell\}_{\ell=1}^L\). \\
\(B_w,B_a\) & Weight and activation bit-widths. \\
\(B\) & Total bit budget, $B=L\bar b$. \\
\(c_\ell(b_\ell)\) & Computational cost of layer \(\ell\) under bit-width \(b_\ell\). \\
\(C_{\mathrm{APQ}}(\boldsymbol b,\boldsymbol{\phi})\) & Expected bit operations per input under allocation \(\boldsymbol b\) and exit policy \(\boldsymbol{\phi}\). \\

\(z_\ell(x)\) & {Full-precision logits} at layer \(\ell\). \\
\(\rho_\ell(x)\) & {Full-precision confidence}, \(\rho_\ell(x)=\max_j \mathrm{softmax}(z_\ell(x))_j\). \\
{\(\tilde\rho_\ell(x)\)} & {Quantized confidence; shorthand for $\tilde\rho_{\ell,\boldsymbol b}(x)$.} \\
\(\phi_\ell\) & Confidence threshold at {exit layer \(\ell\in\mathcal E\)}. \\
\(\boldsymbol{\phi}\) & Exit policy, {\(\boldsymbol{\phi}=\{\phi_e\}_{e\in\mathcal E}\)}. \\
\(K_{\boldsymbol{\phi}}(x)\) & Stopping depth induced by the full-precision exit policy. \\
\(\widetilde K_{\boldsymbol{\phi},\boldsymbol b}(x)\) & Stopping depth after quantization. \\
\(u_\ell(x,\boldsymbol{\phi})\) & Indicator that input \(x\) reaches layer \(\ell\). \\
\(u_\ell(\boldsymbol{\phi})\) & Expected utilization of layer \(\ell\), \(\mathbb E_{x\sim\mathcal D}[u_\ell(x,\boldsymbol{\phi})]\). \\
\(\Delta_\ell(x)\) & Confidence margin, \(\Delta_\ell(x)=\rho_\ell(x)-\phi_\ell\). \\
\(\xi_{\ell,\boldsymbol b}(x)\) & Quantization-induced confidence perturbation. \\

\(h_\ell(x)\) & Full-precision hidden state at layer \(\ell\). \\
\(\tilde h_{\ell,\boldsymbol b}(x)\) & Quantized hidden state at layer \(\ell\). \\
\(\varepsilon_\ell(b_\ell)\) & Local quantization error bound at layer \(\ell\). \\
\(\gamma_\ell\) & Local error amplification factor. \\
\(S_\ell\) & Static quantization sensitivity of layer \(\ell\). \\

\(\mathrm{PGR}_\ell\) & Excess loss relative to the full-depth output at layer $\ell$. \\
\(\mathrm{BSR}_\ell\) & Boundary sensitivity risk at layer \(\ell\). \\
\(\mathrm{QID}_\ell\) & Quantization-induced drift at layer \(\ell\). \\
\(\sigma_\ell,\tilde{\sigma}_\ell\) & Full-precision and quantized residual-update magnitudes. \\
\(\mathrm{SQNR}^{-1}_\ell\) & Inverse signal-to-quantization-noise ratio at layer \(\ell\). \\
\(R_\ell\) & Overall risk score of layer \(\ell\). \\
\(R_{\mathrm{exit}}\) & Ordered pair of policy boundary sensitivity and exit prediction loss. \\
\(w_j\) & Weight of risk component $j$. \\
\(U_\ell(b)\) & Accumulated estimated risk reduction at bit-width $b$. \\
\(\Delta R_\ell(b)\) & Estimated risk reduction from increasing $b$ by one bit. \\
\(\Delta\mathrm{BOPs}_\ell\) & Amortized cost increase from one additional bit at layer $\ell$. \\
\(\boldsymbol b^{+\ell}\) & Allocation with only layer $\ell$ increased by one bit. \\
\(\lambda\) & Trade-off coefficient in the APQ objective. \\
\(\lambda_{\mathrm{outer}}\) & Trade-off coefficient for threshold optimization. \\
\(\lambda_{\mathrm{inner}}\) & Trade-off coefficient for bit-width allocation. \\

\(\mathcal E\) & Set of intermediate exit layers. \\
\(\mathcal T_\phi\) & {Set of candidate exit policies.} \\
\(s_\ell\) & Normalized layer sensitivity for precision allocation. \\
\(\bar b\) & Prescribed average weight precision. \\
\(\widehat{\mathrm{Acc}}(\boldsymbol b,\boldsymbol\phi)\) & Estimated accuracy on the calibration set. \\
\(A_{\mathrm{target}}\) & Target accuracy for threshold and model selection. \\
\(\hat L\) & Exit layer required to reach the target performance. \\
\(\bar L\) & Average exit layer. \\
\(\widehat{\mathrm{BOPs}}\) & BOPs required to reach the target performance. \\
\(\overline{\mathrm{BOPs}}\) & Average BOPs under the standard configuration. \\

\bottomrule
\end{tabular}

\end{table}
\end{document}